\documentclass{article} % For LaTeX2e
\usepackage{nips13submit_e,times}
\nipsfinalcopy % Uncomment for camera-ready version

\newcommand{\Cpp}{C\kern-0.05em\texttt{+\kern-0.03em+}}

\newcommand{\ConceptCpp}{ConceptC\kern-0.05em\texttt{+\kern-0.03em+}}

\usepackage{graphicx} % more modern
\usepackage{subfigure}

\usepackage[numbers]{natbib}
\usepackage{setspace}
\usepackage{algorithm}
\usepackage{algorithmic}

\usepackage{amsmath,bbm,amssymb,array}
\usepackage{color}

\usepackage{hyperref}
\usepackage{qtree}

\usepackage{epsf}
\usepackage{standard}
\usepackage{epsfig}
\usepackage{eufrak}
\long\def\COMMENT#1{}
\DeclareMathAlphabet{\mathpzc}{OT1}{pzc}{m}{it}

\usepackage[compact]{titlesec}
\titlespacing{\section}{0pt}{*0}{*0}

\usepackage{bm}

\def\sS{{\mathcal S}}
\def\sT{{\mathcal T}}
\def\sL{{\mathcal L}}

\def\bmu{{\bm \mu}}

\def\bSigma{{\bm \Sigma}}

\def\bg{{\bm g}}

\def\bxi{{\bm \xi}}
\def\bI{{\bm I}}
\def\bS{{\bm S}}

\def\bzero {{\bm 0}}

\def \btheta {\bm \theta} 
\def \bp {\bm p} 
\def \bg {\bm g} 
\def \bs {\bm s} 
 
\def \bw {\bm w} 
\def \bx {\bm x}

\renewcommand{\[}{\begin{eqnarray*}}
\renewcommand{\]}{\end{eqnarray*}}

\newtheorem{theorem}{Theorem}
\newtheorem{lemma}{Lemma}

\newtheorem{proof}{Proof}

\title{Semistochastic quadratic bound methods}

\author{
Aleksandr Aravkin\\
IBM T.J. Watson Research Center\\
Yorktown Heights, NY 10598 \\
\texttt{saravkin@us.ibm.com} \\
\And
Anna Choromanska \\
Columbia University \\
NY, USA \\
\texttt{aec2163@columbia.edu} \\
\And
Tony Jebara \\
Columbia University \\
NY, USA \\
\texttt{jebara@cs.columbia.edu} \\
\And
Dimitri Kanevsky \\
IBM T.J. Watson Research Center \\
Yorktown Heights, NY 10598 \\
\texttt{dimitri.kanevsky@gmail.com} \\
}

\begin{document}
\maketitle


%\twocolumn[

\begin{abstract}
Partition functions arise in a variety of settings, including conditional random fields, logistic regression, and latent 
gaussian models. In this paper, we consider semistochastic quadratic bound (SQB) methods for maximum likelihood
estimation based on partition function optimization. Batch methods based on the quadratic bound 
were recently proposed for this class of problems, and performed favorably
in comparison to state-of-the-art techniques. Semistochastic methods fall in between batch algorithms, which use all the data, 
and stochastic gradient type methods, which use small random selections at each iteration. 
We build semistochastic quadratic bound-based methods, and prove both global convergence (to a stationary point)
under very weak assumptions, and linear convergence rate under stronger assumptions on the objective. To make the proposed methods faster and more stable, we consider inexact subproblem 
minimization and batch-size selection schemes.
The efficacy of SQB methods is demonstrated via comparison with several state-of-the-art techniques on
commonly used datasets. 
\end{abstract}

\section{Introduction}
The problem of optimizing a cost function expressed as the sum of a loss term over each sample in an input dataset is pervasive in machine learning. One example of a cost function of this type is the
partition funtion. Partition function is a central quantity in many different learning tasks including training conditional random fields (CRFs) and log-linear models~\cite{JebCho12}, and will be of central focus in this paper. Batch methods based on the quadratic bound 
were recently proposed~\cite{JebCho12} for the class of problems invoving the minimization of the partition function, and performed favorably
in comparison to state-of-the-art techniques. This paper focuses on semistochastic extension of this recently developed optimization method. Standard learning systems based on batch methods such as  BFGS and memory-limited L-BFGS, steepest descent (see e.g.~\cite{opac-b1104789}), conjugate gradient~\cite{Hestenes&amp;Stiefel:1952} or quadratic bound majorization method~\cite{JebCho12} need to make a full pass through an entire dataset before updating the parameter vector. Even though these methods can converge quickly (sometimes in several passes through the dataset),
 as datasets grow in size, this learning strategy becomes increasingly inefficient. To faciliate learning on massive datasets, the community increasingly turns to stochastic methods. 

Stochastic optimization methods %, which go far back to~\cite{RobbinsMunro1951}, 
interleave the update of parameters after only processing a small mini-batch of examples 
(potentially as small as a single data-point), leading to significant computational savings~(\cite{bottou-98x,Littlestone:1988:LQI:639961.639994,rosenblatt58a}). Due to its  simplicity and low computational cost, the most popular contemporary stochastic learning technique is stochastic gradient descent (SGD) ~\cite{citeulike:432261,bottou-98x,DBLP:conf/nips/BottouC03}. SGD updates the parameter vector using the gradient of the objective function as evaluated on a single example (or, alternatively, a small mini-batch of examples). This algorithm admits multiple extensions, including (i) stochastic average gradient method (SAG) that averages the most recently computed gradients for each training example~\cite{DBLP:conf/nips/RouxSB12}, (ii) methods that compute the (weighted) average of all previous gradients~\cite{DBLP:journals/mp/Nesterov09,Tseng:1998:IGM:588881.588930}, (iii) averaged stochastic gradient descent method (ASGD) that computes a running average of parameters obtained by SGD~\cite{Polyak:1992:ASA:131092.131098}, (iv) stochastic dual coordinate ascent, that optimizes the dual objective with respect to a single dual vector or a mini-batch of dual vectors chosen uniformly at random~\cite{DBLP:journals/corr/abs-1211-2717,DBLP:journals/corr/abs-1305-2581}, (v) variance reduction techniques~\cite{NIPS2013_4937,NIPS2013_5034,DBLP:conf/nips/RouxSB12,DBLP:journals/corr/abs-1209-1873} (some do not require storage of gradients, c.f.~\cite{NIPS2013_4937}), (vi) majorization-minimization techniques that minimize a majoring surrogate of an objective function~\cite{DBLP:conf/icml/Mairal13,NIPS2013_5129} and (vii) gain adaptation techniques~\cite{Schraudolph99localgain, Schraudolph02fastcurvature}. 

Semistochastic methods can be viewed as an interpolation between the expensive reliable updates used by full batch methods, 
and inexpensive noisy updates used by stochastic methods. 
They inherit the best of both worlds by approaching the solution more quickly when close to the optimum (like a full batch method) while simultaneously reducing the computational complexity per iteration (though less aggressively than stochastic methods). 
Several semistochastic extensions have been explored in previous works~\cite{ShapiroWardi,Shapiro00onthe,Kleywegt:2002:SAA:588882.588955}. Recently, convergence theory and sampling strategies for these methods have been explored in~\cite{FriedlanderSchmidt2012,Byrd:2012:SSS:2348131.2348140} and linked to results in finite sampling theory in~\cite{AravkinFHV:2012}.
 
Additionally, incorporating
second-order information (i.e. Hessian) into the optimization problem (\cite{Schraudolph99localgain,Schraudolph07astochastic,lecun-98x, Bordes:2009:SCQ:1577069.1755842,DBLP:conf/icml/Martens10,ByrdCNN11}) was shown to often improve the performance of traditional SGD methods which typically provide fast improvement initially, 
but are slow to converge near the optimum (see e.g.~\cite{FriedlanderSchmidt2012}), require step-size
tuning and are difficult to parallelize
\cite{DBLP:conf/icml/LeNCLPN11}. This paper focuses on semistochastic extension of a recently developed quadratic bound majorization technique~\cite{JebCho12}, and we call the new algorithm {\it semistochastic quadratic bound} (SQB) method. The bound computes the update on the parameter vector using the product of the gradient of the objective function and an inverse of a second-order term that is a descriptor of the curvature of the objective function (different than the Hessian). We discuss implementation details, in particular curvature approximation, inexact solvers, and batch-size selection strategies, which make the running time of our algorithm comparable to the gradient methods and also make the method easily parallelizable. We show global convergence of the method to a stationary point under very weak assumptions (in particular convexity is not required) and a linear convergence rate when the size of the mini-batch grows sufficiently fast, following the techniques of~\cite{FriedlanderSchmidt2012}. 
This rate of convergence matches state-of-the-art incremental techniques~\cite{NIPS2013_4937,DBLP:journals/corr/abs-1211-2717,DBLP:conf/nips/RouxSB12,DBLP:conf/icml/Mairal13} (furthermore it is better than in case of standard stochastic gradient methods~\cite{citeulike:432261,DBLP:conf/nips/BottouC03} which typically have sublinear convergence rate~\cite{DBLP:conf/nips/RouxSB12,journals/tit/AgarwalBRW12}). Compared to other existing majorization-minimization incremental techniques~\cite{DBLP:conf/icml/Mairal13}, our approach uses much tighter bounds which, as shown in~\cite{JebCho12}, can lead to faster convergence.

The paper is organized as follows: Section~\ref{sec:qb} reviews quadratic bound majorization technique. Section~\ref{sec:stochastics} discusses stochastic 
and semistochastic extensions of the bound, and presents convergence theory for the proposed methods. In particular, 
we discuss very general stationary convergence theory under very weak assumptions, and also present 
a much stronger theory, including convergence rate analysis, for logistic regression. Section~\ref{sec:imp} discusses implementation details, and Section~\ref{sec:convex} shows
numerical experiments illustrating the use of the proposed methods for $l_2$-regularized logistic regression problems. Conclusions end the paper. 

The semistochastic quadratic bound majorization technique that we develop in this paper can be broadly applied to mixture models or models that induce representations. The advantages of this technique in the batch setting for learning mixture models and other latent models, was shown in the work of~\cite{JebCho12}. 
In particular, quadratic bound majorization was able to find better local optima in non-convex problems than state-of-the art methods (and in less time). 
While theoretical guarantees for non-convex problems are hard to obtain, the broader convergence theory developed in this paper 
(finding a stationary point under weak assumptions) does carry over to the non-convex setting.

%%%%%%%%%%%%%%%%%%%%%%%%%%%%%%%%%%%%%%%%%%%%%%%%%%%
\section{Quadratic bound methods }
\label{sec:qb}
%%%%%%%%%%%%%%%%%%%%%%%%%%%%%%%%%%%%%%%%%%%%%%%%%%%

Let $\Omega$ be a discrete probability space over the set of $n$ elements, 
and take any log-linear density model
\begin{equation}
\label{discreteModel}
p(y|x_j,\btheta) = \frac{1}{Z_{x_j}(\btheta)} h_{x_j}(y) \exp \left ( \btheta^\top {\bf f}_{x_j}(y) \right )
\end{equation}
parametrized by a vector $\btheta \in \mathbb{R}^d$, where $\{(x_1,y_1), \dots, (x_T,y_T)\}$ are iid input-output pairs, 
${\bf f}_{x_j}:\Omega \mapsto
\mathbb{R}^d$ is a continuous vector-valued function mapping and $h_{x_j}: \Omega \mapsto
\mathbb{R}^{+}$ is a fixed non-negative measure.  The {\it partition function} $Z_{x_j}(\btheta)$ is a
scalar that ensures that $p(y|x_j,\btheta)$ is a true density, so in particular~\eqref{discreteModel}
integrates to $1$:
\begin{equation}
\label{partitionFunction}
 Z_{x_j}(\btheta)=\sum_y h_{x_j}(y) \exp(\btheta^\top {\bf f}_{x_j}(y))\;.
\end{equation} 

\cite{JebCho12} propose a fast method to find a tight quadratic bound for $Z_{x_j}(\btheta)$, 
shown in the subroutine Bound Computation in Algorithm~\ref{boundAlgorithm}, 
which finds $z,{\bf r},{\bf S}$ so  that
\begin{equation}
\label{boundProperty}
Z_{x_j}(\btheta) \leq z \exp  ( \tfrac{1}{2}
  (\btheta-{\tilde \btheta})^\top {\bf S} (\btheta-{\tilde \btheta}) +
  (\btheta-{\tilde \btheta})^\top\! {\bf r}  )
  \end{equation}
for any
$\btheta,{\tilde \btheta}, {\bf f}_{x_j}(y) \in \mathbb{R}^d$ and $h_{x_j}(y) \in
\mathbb{R}^+$ for all $y \in \Omega$.

The (regularized) maximum likelihood estimation problem is equivalent to   
\begin{equation}
\label{minLog}
\begin{aligned}
\min_{\btheta} &\Big\{\sL_\eta(\btheta) := -\frac{1}{T}\sum_{j=1}^{T}\log(p(y_j|x_j,\btheta)) + \frac{\eta}{2} \|\btheta\|^2 \approx \frac{1}{T}\sum_{j=1}^{T} \left (\log(Z_{x_j}(\btheta)) - \btheta^\top {\bf f}_{x_j}(y_j) \right ) + \frac{\eta}{2} \|\btheta\|^2\Big\}\;,
\end{aligned}
\end{equation}
where $\approx$ means equal up to an additive constant.
The bound~\eqref{boundProperty} suggests the iterative minimization scheme
%\begin{equation}
%\label{boundIteration}
%\btheta^{k+1} = \btheta^k - \alpha_k\bSigma^{-1}\bmu, 
%\end{equation}
%In practice, state of the art methods often use $\ell_2$ regularization, 
%adding a term $\lambda \|\btheta\|^2$ to the objective~\eqref{minLog}.
%For the bound method, this yields the regularized iteration 
\begin{equation}
\label{regBoundIteration}
\btheta^{k+1} = \btheta^k - \alpha_k(\bSigma^k + \eta \bI)^{-1}(\bmu^k + \eta\btheta^k).
\end{equation}
where $\bSigma^k$ and $\bmu^k$ are computed using Algorithm~\ref{boundAlgorithm}, $\eta$ is the regularization term and $\alpha_k$ is the step size at iteration $k$.

In this paper, we consider applying the bound to randomly selected batches of data; 
any such selection we denote $\mathcal{T} \subset [1, \dots, T]$ or $\sS \subset[1, \dots, T]$.

\begin{algorithm}
\caption{Semistochastic Quadratic Bound (SQB)\label{boundAlgorithm}}
% \begin{tabular}{ll}
% \mbox{
\begin{tabular}{l}
\vspace{-0.185in} \\
\hline
Input Parameters ${\tilde \btheta}, {\bf f}_{x_j}(y) \in \mathbb{R}^d$ and $h_{x_j}(y) \in \mathbb{R}^+$ for $y \in \Omega$, $j \in \mathcal{T}$ \:\:\:\:\:\:\:\:\:\:\:\:\:\:\:\:\:\:\:\:\:\:\:\:\:\:\:\:\:\:\:\:\:\:\:\:\:\:\:\:\:\:\:\:\:\:\:\:\: \\
\hline
Initialize $\bmu_\sT = \bzero, \bSigma_\sT = \bzero(d,d)$\\
% \!\!2:\!\! & \!\!\! Choose random $\pi:\Omega \mapsto \{1,\ldots,n\}$ where $n=|\Omega|$ \\
For each $j \in \mathcal{T}$ \\
\\
\:\:\:\:\:\:\: Subroutine \textbf{Bound Computation:}\\
\:\:\:\:\:\:\: $z \rightarrow 0^+, {\bf r} = \bzero, {\bf S}=z{\bf I}$\\
\:\:\:\:\:\:\: For each $y \in \Omega$ \\
$\:\:\:\:\:\:\:\:\:\:\:\:\:\: \alpha = h_{x_j}(y) \exp({\tilde \btheta}^\top {\bf f}_{x_j}(y)) $\\
$\:\:\:\:\:\:\:\:\:\:\:\:\:\: {\bf S} +\!= \frac{\tanh(\frac{1}{2} \log (\alpha/z))}{2 \log ( \alpha/z)} ({\bf f}_{x_j}(y)-{\bf r})({\bf f}_{x_j}(y)-{\bf r})^\top \:\:\: $\\
$\:\:\:\:\:\:\:\:\:\:\:\:\:\: {\bf r} = \frac{z}{z+\alpha }{\bf r} + \frac{\alpha}{z+\alpha} {\bf f}_{x_j}(y) $ \\
$\:\:\:\:\:\:\:\:\:\:\:\:\:\: z +\!=\alpha$ \\
\:\:\:\:\:\:\: Subroutine output $z,{\bf r},{\bf S}$ \\
\\
$\:\:\:\:\:\:\: \bmu_\sT +\!= {\bf r} - {\bf f}_{x_j}(y)$\\
$\:\:\:\:\:\:\: \bSigma_\sT +\!= {\bf S}$\\
$\bmu_\sT /\!= |\sT|$\\
$\bSigma_\sT /\!= |\sT|$\\
\hline
Output $\bmu_\sT,\bSigma_\sT$ \\
\hline
\vspace{-0.19in} 
\end{tabular}
% }
% &  \epsfysize=1.5in \epsfbox{exampleBounds.eps} \end{tabular} \vspace{-0.85in}
%\textcolor{red}{In the above algorithm shouldn't it be $\bmu_\sT +\!= {\bf r} - {\bf f}_{x_j}(y) + \frac{2\eta}{|\mathcal{T}|}\tilde{\btheta}$? If yes the following analysis must change}
\end{algorithm}

%\textcolor{red}{Sasha:  $\Omega$ is the label space and $\mathcal{T}$ is the input dataset. We sample $\mathcal{T}$. Now if the cardinality of $\mathcal{T}$ (the size of input dataset) is $T$ and the cardinality of $\Omega$ (the number of labels) is $n$, then computing the bound in total takes $O(Tnd^2)$, since for each data point we compute the bound.}

%%%%%%%%%%%%%%%%%%%%%%%%%%%%%%%%%%%%%%%%%%%%%%%%%%%
\section{Stochastic and semistochastic extensions}
\label{sec:stochastics}
%%%%%%%%%%%%%%%%%%%%%%%%%%%%%%%%%%%%%%%%%%%%%%%%%%%

The bounding method proposed in~\cite{JebCho12} is summarized in Algorithm~\ref{boundAlgorithm} with  
$\sT = [1, \dots, T]$ at every iteration. 
When $T$ is large, this strategy can be expensive. In fact, 
computing the bound has complexity $O(Tnd^2)$, since $Tn$ outer products must be summed to obtain $\bSigma$, and 
each other product has complexity $O(d^2)$. When the dimension $d$ is large, considerable speedups can be gained by 
obtaining a factored form of $\bSigma$, as described in Section~\ref{sec:Efficient}. Nonetheless, in either strategy, 
the size of $T$ is a serious issue. 
 
A natural approach is to subsample a smaller selection $\sT$ from the training set $[1, \dots, T]$, so that at each iteration, we run Algorithm~\ref{boundAlgorithm}
over $\sT$ rather than over the full data to get $\bmu_{\sT},\bSigma_{\sT}$. 
When $|\sT|$ is fixed (and smaller than $T$), we refer to the resulting method as a {\it stochastic} extension. 
If instead $|\sT|$ is allowed to grow as iterations proceed, we call this method {\it semistochastic}; these methods
are analyzed in~\cite{FriedlanderSchmidt2012}. All of the numerical experiments we present focus on  semistochastic methods. 
One can also decouple the computation of gradient and curvature  approximations, 
using different data selections (which we call $\sT$ and $\sS$). We show that this development 
is theoretically justifiable and practically very useful. 

For the stochastic and semistochastic methods discussed here, 
the quadratic bound property~\eqref{boundProperty} does not hold for $Z(\btheta)$, 
so the convergence analysis of~\cite{JebCho12} does not immediately apply. Nonetheless, it is possible to 
analyze the algorithm in terms of sampling strategies for $\sT$. 

The appeal of the stochastic modification is that when $|\sT| << T$, the complexity  $O(|\sT| n d)$ 
of Algorithm~\ref{boundAlgorithm} to compute $\bmu_{\sT},\bSigma_{\sS}$ is much lower; 
 and then we can still implement a (modified)
iteration~\eqref{regBoundIteration}. Intuitively, one expects that even small samples from the data 
can give good updates for the overall problem. This intuition is supported by the experimental results, which show 
that in terms of effective passes through the data, SQB is competitive with state of the art methods.  
%make consistently more progress {\it per iteration} than a quasi-newton method relying on full data. 

We now present the theoretical analysis of Algorithm~\ref{boundAlgorithm}.
We first prove that under very weak assumption, in particular using only the Lipschitz property,  
but not requiring convexity of the problem, the proposed algorithm converges to a stationary point. 
The proof technique easily carries over to other objectives, such as the ones used in maximum latent conditional likelihood problems (for details see~\cite{JebCho12}),
since it relies mainly only on the sampling method used to obtain $\sT$.  
Then we focus on problem~\eqref{minLog}, which is convex, and strictly convex under appropriate assumptions on the data. 
We use the structure of~\eqref{minLog} to prove much stronger results, and in particular analyze the rate of convergence of Algorithm~\ref{boundAlgorithm}. 

%%%%%%%%%%%%%%%%%%%%%%%%%%%%%%%%%%%%%%%%%%%%%%%%%%
\subsection{General Convergence Theory}
%%%%%%%%%%%%%%%%%%%%%%%%%%%%%%%%%%%%%%%%%%%%%%%%%%

We present a general global convergence theory, that relies  on the Lipschitz  
property of the objective and on the sampling strategy in the context of Algorithm~\ref{boundAlgorithm}. 
The end result we show here is that any limit point of the iterates is {\it stationary}. 
We begin with two simple preliminary results. 

\begin{lemma}
\label{ExpectationLemma}
If every $i \in [1, \dots, T]$ is equally likely to appear in $\sT$, then $E[\bmu_{\sT}] = \bmu$.
\end{lemma}
\begin{proof}
Algorithm~\ref{boundAlgorithm} returns 
\(
\bmu_{\sT} = \frac{1}{|\sT|}\sum_{j \in \sT} \psi_j(\btheta)
\), where $\psi_j(\btheta) =  -\nabla_{\btheta}\log(p(y_j|x_j,\btheta))$.
If each 
$j$ has an equal chance to appear in $\sT$, then 
\[
E\left[\frac{1}{|\sT|}\sum_{j \in \sT} \psi_j(\btheta) \right] 
= 
\frac{1}{|\sT|}\sum_{j \in \sT} E[\psi_j(\btheta)]
= \frac{1}{|\sT|}\sum_{j \in \sT}\bmu = \bmu\;.
\]
\end{proof}
Note that the hypothesis here is very weak: there is no stipulation that the batch size 
be of a certain size, grow with iterations, etc. This lemma therefore applies to  
a wide class of randomized bound methods. 

\begin{lemma}
\label{ExpectedSigmaLemma}
Denote by $\lambda_{\min}$ the infimum over all possible eigenvalues of $\bSigma_{\sS}$ over 
all choices of batches ($\lambda_{\min}$ may be $0$). 
Then  $E[(\bSigma_{\sS} + \eta \bI)^{-1}]$ satisfies
\[
\frac{1}{\eta + \lambda_{\max}} \bI \leq E[(\bSigma_{\sS} + \eta \bI)^{-1}] \leq \frac{1}{\eta + \lambda_{\min}} \bI\;.
\]
\end{lemma}
\begin{proof}
For any vector $\bx$ and any realization of $\bSigma_{\sS}$, we have 
\[
\frac{1}{\eta + \lambda_{\max}}\|\bx\|^2 \leq \bx^T(\bSigma_{\sS} + \eta \bI)^{-1}\bx \leq \frac{1}{\eta + \lambda_{\min}}\|\bx\|^2\;,
\]
where $\lambda_{\max}$ depends on the data. Taking the expectation over $\sT$ of be above inequality gives the result. 
\end{proof}

\begin{theorem}
\label{globalConvergenceTheorem}
For any problem of form~\eqref{minLog}, apply iteration~\eqref{regBoundIteration}, 
where at each iteration $\bmu_{\sT}, \bSigma_{\sS}$ are obtained by Algorithm~\ref{boundAlgorithm} 
for two independently drawn batches subsets $\sT, \sS \subset [1, \dots, T]$  selected to satisfy the assumptions
of Lemma~\ref{ExpectationLemma}. % and Lemma~\ref{ExpectedSigmaLemma}. 
Finally, suppose also that the step sizes $\alpha_k$ are 
square summable but not summable. Then $\sL_\eta(\btheta^k)$ converges to a finite value, 
and $\nabla \sL_\eta(\btheta^k) \rightarrow 0$. Furthermore, every limit point of $\btheta^k$
is a stationary point of $\sL_\eta$. 
\end{theorem}

Theorem~\ref{globalConvergenceTheorem} states the conclusions of~\cite[Proposition 3]{Bertsekas99gradientconvergence}, 
and so to prove it we need only check that the hypotheses of this proposition are satisfied. 

\begin{proof}
%We have shown that $g$ is Lipschitz continuous and bounded below. 
\cite{Bertsekas99gradientconvergence} consider algorithms of the form 
\[
\btheta^{k+1} = \btheta^{k} - \alpha_k (\bs^k + \bw^k)\;. 
\]
In the context of iteration~\eqref{regBoundIteration}, at each iteration we have 
\[
\bs^k + \bw^k = (\bSigma_{\sS}^k + \lambda \bI)^{-1} \bg_{\sT}^k,
\]
where $\bg_{\sT}^k = \bmu_{\sT}^k + \eta \btheta^k$, and $\bg^k$ is the full gradient of the regularized problem~\eqref{minLog}. 
We choose 
\[
\bs^k &=  E[(\bSigma_{\sS}^k+\eta \bI)^{-1}] \bg^k, \quad \bw^k =  (\bSigma_{\sS}^k + \eta \bI)^{-1} \bg_{\sT}^k- \bs^k.
\]
We now have the following results: 
\begin{enumerate}
\item Unbiased error: 
\begin{equation}
\label{unbiased}
\begin{aligned}
E[\bw^k] &= E[(\bSigma_{\sS}^k+\eta \bI)^{-1} \bg_{\sT}^k - \bs^k] =  E[(\bSigma_{\sS}^k + \eta \bI)^{-1}]E[\bg_{\sT}^k] - \bs^k= 0\;,
\end{aligned}
\end{equation}
where the second equality is obtained by independence of the batches $\sT$ and $\sS$, and the last equality
uses Lemma~\ref{ExpectationLemma}. 
\item Gradient related condition: 
\begin{equation}
\label{MinDescent}
\begin{aligned}
(\bg^k)^T\bs^k &= (g^k)^TE[(\bSigma_{\sS}^k + \eta \bI)^{-1}] \bg^k
& \geq \frac{\|\bg^k\|^2}{\eta+\lambda_{\max}}.
\end{aligned}
\end{equation}
%by Lemma~\ref{ExpectedSigmaLemma}. 
\item Bounded direction:
\begin{equation}
\label{BoundedDirection}
\|\bs^k\| \leq \frac{\|\bg^k\|}{\eta+\lambda_{\min}} .
\end{equation}
%where $\lambda_{\min}$ comes from Lemma~\ref{ExpectationLemma}.
\item Bounded second moment: 

By part 1, we have 
\begin{eqnarray}
\label{BoundedMomentSimple}
\begin{split}
E[\|\bw^k\|^2]  &\leq  E[\|(\bSigma_{\sS}^k + \eta \bI)^{-1} \bg_{\sT}^k\|^2\\ %+ \|\bs^k\|^2\\
& \leq \frac{E[\|\bg_{\sT}^k\|^2]  }{(\eta + \lambda_{\min})^2}  = \frac{\text{tr}(\text{cov}[\bg_{\sT}^k]) +\|\bg^k\|^2 }{(\eta + \lambda_{\min})^2}.
\end{split}
\end{eqnarray}
%
%We have 
%\begin{eqnarray}
%\label{BoundedMoment}
%\begin{split}
%E[\|\bw^k\|^2] &= E[\|(\bSigma_{\sS}^k + \eta \bI)^{-1} \bg_{\sT}^k - \bs^k\|^2]\\
%&\leq E[\|(\bSigma_{\sS}^k+\eta \bI)^{-1} \bg_{\sT}^k\|^2 + 2|(\bs^k)^T(\bSigma_{\sS}^k+ \eta \bI)^{-1} \bg_{\sT}^k| 
%+ \|\bs^k\|^2]\\ 
%&\leq \frac{3\|\bg^k\|^2 + 2 tr(cov(\bg_\sT^k))}{(\eta + \lambda_{\min})^2} \;.
%\end{split}
%\end{eqnarray}
\end{enumerate}
%Note that the trace of the covariance matrix of $\bg_\sT^k$ is proportional to the trace of the covariance matrix of $\bg^k$
%and so in particular is finite. 
The covariance matrix of $\bg_\sT^k$ is proportional to the covariance matrix of the set of individual (data-point based) gradient contributions, 
and for problems of form~\eqref{minLog} these contributions lie in the convex hull of the data, so in particular the trace of the covariance must be finite.  
Taken together, these results show all hypotheses of~\cite[Proposition 3]{Bertsekas99gradientconvergence} 
are satisfied, and the result follows. 
\end{proof}

Theorem~\eqref{globalConvergenceTheorem} applies to any stochastic and semistochastic variant
of the method.  Note that two independent data samples $\sT$ and $\sS$ are required to prove~\eqref{unbiased}. 
%We implemented this variant, and did not see a significant difference from the model algorithm, where the same 
%selection is used to construct $\bmu_{\sT}$ and $\bSigma_{\sT}$. 
Computational complexity motivates different strategies for selecting choose $\sT$ and $\sS$.
In particular, it is natural to use larger mini-batches to estimate the gradient, 
and smaller mini-batch sizes for the estimation of the second-order curvature term.
Algorithms of this kind have been explored in the context of stochastic Hessian methods~\cite{ByrdCNN11}. 
We describe our implementation details in Section~\ref{sec:imp}. 

%%%%%%%%%%%%%%%%%%%%%%%%%%%%%%%%%%%%%%%%%%%%%%%%%%
\subsection{Rates of Convergence for Logistic Regression}
%%%%%%%%%%%%%%%%%%%%%%%%%%%%%%%%%%%%%%%%%%%%%%%%%%

The structure of objective~\eqref{minLog} allows for a much stronger convergence theory. 
We first present a lemma characterizing strong convexity and Lipschitz constant for~\eqref{minLog}. 
Both of these properties are crucial to the convergence theory. 

\begin{lemma}
\label{lipschitzLemma}
The objective $\sL_\eta$ in~\eqref{minLog} has a gradient that is uniformly norm bounded,  
and Lipschitz continuous. 
\end{lemma}

\begin{proof}
The function $\sL_\eta$ has a Lipschitz continuous gradient if there exists an $L$
such that 
\[
\|\nabla \sL_\eta(\btheta^1) - \nabla \sL_\eta(\btheta^0)\|\leq L \|\btheta^1 - \btheta^0\|
\]
holds for all $(\btheta^1, \btheta^0)$. 
Any uniform bound for $\text{trace}(\nabla^2 \sL_\eta)$ is a Lipschitz bound for $\nabla \sL_\eta$. 
Define 
\[
a_{y,j} := h_{x_j}(y) \exp(\btheta^\top {\bf f}_{x_j}(y))\;,
\]
and note $a_{y,j} \geq 0$. Let $\bp_j$ be the empirical density 
where the probability of observing $y$ is given by $\frac{a_{y,j}}{\sum_{y} a_{y,j}}$.
The gradient of~\eqref{minLog} is given by 
\begin{equation}
\label{gradient}
\begin{aligned}
\frac{1}{T}\sum_{j=1}^{T}& \left(\left ( \sum_{y}\frac{a_{y,j}{\bf f}_{x_j}(y)}
{\sum_y a_{y,j}}\right ) -  {\bf f}_{x_j}(y_j)\right) + \eta\btheta=
\frac{1}{T}\sum_{j=1}^{T} \left(E_{\bp_j}[{\bf f}_{x_j}(\cdot)] -  {\bf f}_{x_j}(y_j)\right) + \eta\btheta
\end{aligned}
\end{equation}
It is straightforward to check that the Hessian is given by 
\begin{equation}
\label{hessian}
\nabla^2 \sL_\eta = \frac{1}{T}\sum_{j=1}^{T} cov_{\bp_j}[{\bf f}_{x_j}(\cdot)] + \eta \bI
\end{equation}
where $cov_{\bp_j}[\cdot]$ denotes the covariance matrix with respect to the empirical density
function $\bp_j$. 
Therefore a global bound for the Lipschitz constant $L$ 
is given by $\max_{y,j} \|{\bf f}_{x_j}(y)\|^2 + \eta \bI$, which completes the proof. 
\end{proof}

Note that %The function $\sL_0$ is strongly convex exactly when $\sum_{j, y} {\bf f}_{x_j}(y){\bf f}_{x_j}(y)^T$ is positive definite, and 
$\sL_\eta$ is strongly convex for any positive $\eta$.
We now present a convergence rate result, using results from~\cite[Theorem 2.2]{FriedlanderSchmidt2012}.
\begin{theorem}
\label{rateConvergenceTheorem}
%Suppose that $g$ is strongly convex, with $\mu I \leq \nabla^2 g \leq L I$. 
%Suppose that the selection strategy of $\sS$ 
%satisfies Lemma~\ref{ExpectedSigmaLemma}. Then 
There exist $\mu, L > 0, \rho>0$
such that 
\begin{equation}
\label{Conditions}
\begin{aligned}
\|\nabla \sL_\eta(\btheta_1) - \nabla \sL_\eta(\btheta_2)\|_{**} \leq L\|\btheta_2-\btheta_1\|_{*}\\
\sL_\eta(\btheta_2) \geq \sL_\eta(\btheta_1) + (\btheta_2 - \btheta_1)^T\nabla \sL_\eta(\btheta_1) + \frac{1}{2}\rho \|\btheta_2-\btheta_1\|_{*}
\end{aligned}
\end{equation}
 where $\|\btheta\|_* = \sqrt{\btheta^T (\bSigma_{\sS}^k+\eta \bI)\btheta}$ and $\|\btheta\|_{**}$ 
 is the corresponding dual norm $\sqrt{\btheta^T (\bSigma_{\sS}^k + \eta \bI)^{-1}\btheta}$.  
Furthermore, take $\alpha_k = \frac{1}{L}$ in~\eqref{regBoundIteration}, and define 
$B_k = \|\nabla \sL_\eta^k - \bg_\sT^k\|^2$, the square error incurred in the gradient at iteration $k$. 
Provided a batch growth schedule with $\lim_{k\rightarrow \infty} \frac{B_{k+1}}{B_k}\leq 1$, 
for each iteration~\eqref{regBoundIteration} we have (for any $\epsilon > 0$)
\begin{equation}
\label{rateBound}
\sL_\eta(\btheta^k) - \sL_\eta(\btheta^*) \leq \left(1-\frac{ \rho}{ L}\right)^k[\sL_\eta(\btheta^0) - \sL_\eta(\btheta^*)] + \mathcal{O}(C_k)\;,
\end{equation}
with $C_k = \max\{ B_k, (1-\frac{ \rho}{ L} + \epsilon)^k\}$.
\end{theorem}

\begin{proof}
Let $\tilde L$ denote the bound on the Lipschitz constant of $g$ is provided in~\eqref{gradient}. 
By the conclusions of Lemma~\ref{ExpectedSigmaLemma}, we can take $L = \frac{1}{\sqrt{\eta + \lambda_{\min}}} \tilde L$.
Let $\tilde \rho$ denote the minimum eigenvalue of~\eqref{hessian} (note that $\tilde\rho \geq \eta)$. 
Then take $\rho = \frac{1}{\sqrt{\eta + \lambda_{\max}}}\tilde \rho$.  
The result follows immediately by~\cite[Theorem 2.2]{FriedlanderSchmidt2012}.  
\end{proof}

%%Sasha

\section{Implementation details}
\label{sec:imp}
In this section, we briefly discuss important implementation details as well as describe the comparator methods we use for our algorithm. 

\subsection{Efficient inexact solvers}
\label{sec:Efficient}

The linear system we have to invert in iteration~\eqref{regBoundIteration} 
has very special structure. The matrix $\bSigma$ returned by Algorithm~\ref{boundAlgorithm}
may be written as $\bSigma = \bS\bS^T$, where each column of $\bS$ is proportional to 
one of the vectors $({\bf f}_{x_j}(y)-{\bf r})$ computed by the bound. When the dimensions of $\btheta$ are large, it is not practical to compute the $\bSigma$ explicitly. 
Instead, to compute the update in iteration~\eqref{regBoundIteration}, we take advantage of the fact that 
\[
\bSigma \bx = \bS(\bS^T\bx),
\]
and use $\bS$ (computed with a simple modification to the bound method) to implement the action of $\bSigma$.
When $\bS \in \mathbb{R}^{d\times k}$ ($k$ is a mini-batch size), 
the action of the transpose on a vector can be computed in $O(dk)$, which is very efficient 
for small $k$. The action of the regularized curvature approximation $\bSigma + \eta \bI$ follows
immediately. Therefore, it is efficient to use iterative minimization schemes, such as \textsf{lsqr},  
conjugate gradient, or others to compute the updates. Moreover, using only a few 
iterations of these methods further regularizes the subproblems~\cite{Martens2012,vogel}.

It is interesting to note that even when $\eta = 0$, and $\bSigma_{\sT}$ is not invertible, 
it makes sense to consider inexact updates.  
To justify this approach, we present a range lemma. A similar lemma appears in~\cite{Martens2012} for 
a different quadratic approximation. 
\begin{lemma}
\label{RangeLemma}
For any $\sT$, we have $\bmu_{\sT} \in \mathcal{R}(\bSigma_{\sT})$.
\end{lemma}
\begin{proof}
The matrix $\bSigma_{\sT}$ is formed by a sum of weighted outer products $({\bf f}_{x_j}(y)-\bm r)({\bf f}_{x_j}(y)-\bm r)^\top$.
%The notation $\bmu_i$ emphasizes that the quantity $\bmu$ is changing with each data point.
We can therefore write
\[
\bSigma_{\sT} = \bm L\bm D\bm L^T
\]
where $\bm L = [\bm l_1, \dots, \bm l_{|\Omega|\cdot|\sT|}]$, $\bm l_k = {\bf f}_{x_j}(y_k)-\bm r^k$ ($k$ is the current iteration of the bound computation), and $\bm D$ is a diagonal matrix with weights 
$\bm D_{kk}=\frac{1}{|\sT|}\frac{\tanh(\frac{1}{2} \log (\alpha_k/z_k))}{2 \log ( \alpha_k/z_k)}$, where the quantities $\alpha_k, z_k$
 correspond to iterations in Algorithm~\eqref{boundAlgorithm}.
Since $\bmu$ is in the range of $\bm L$ by construction, 
it must also be the range of $\bSigma_{\sT}$.
\end{proof}

Lemma~\ref{RangeLemma} tells us that 
there is always a solution to the linear system $\bSigma_{\sT}{\bf \Delta \theta} = \bmu_{\sT}$, 
even if $\bSigma_{\sT}$ is singular. In particular, a minimum norm solution can be found 
using the Moore-Penrose pseudoinverse, or by simply applying 
\textsf{lsqr} or \textsf{cg}, which is useful in practice 
when the dimension $d$ is large. For many problems, using a small number of \textsf{cg} iterations 
both speeds up the algorithm and serves as additional regularization at the earlier iterations, since
the (highly variable) initially small problems are not fully solved. 
%\textcolor{red}{Throughout the entire section you use different notation, once you have $({\bf f}_{x_j}(y)-{\bf r})$ and then $({\bf f}(y_i)-\bm r)$!!!}
\subsection{Mini-batches selection scheme}
In our experiments, we use a simple linear interpolation scheme to grow the batch sizes for both the gradient 
and curvature term approximations. 
In particular, each batch size (as a function of iteration $k$) is given by 
\[
b^k = \min(b^{\text{cap}} , b^1 + \textsf{round}((k-1) \gamma)),
\] 
where $b^\text{cap}$ represents the cap on the maximum allowed size, $b^1$ is the initial batch size, and $\gamma$ gives the rate of increase. 
In order to specify the selections chosen, we will simply give values for each of $(b^1_{\bmu}, b^1_{\bSigma}, \gamma_\bmu, \gamma_\bSigma)$.
For all experiments, the cap $b^\text{cap}_\bmu$ on the gradient computation was the full training set, the cap $b^\text{cap}_\bSigma$
for the curvature term was taken to be 200, initial $b^1_{\bmu}$ and $b^1_{\bSigma}$ were both set to $5$.   
%We use a simple the following mini-batch size selecting scheme for the gradient estimation in SQB method: 
%\[b = \min(b_1,T),
%\]
%where $b_1 = 5 + jb_2(1 + \beta b_2)$, $b_2 = \max(\lceil \gamma_{\bmu} T \rceil,1)$, $j$ is the current iteration and $\beta, \gamma_{\bmu} \in [0,1]$ are tuning parameters. In particular $\gamma_{\bmu}$ determines the rate of increase of the mini-batch size. \textcolor{red}{Explain more and also comment on $beta$. It looks it is difficult to chose those two parameters, maybe there is some hint related to the nature of the dataset? Also this step size choise heuristics looks very complicatied, can we somehow simplify? Also give some intuition of how fast the batch size grows and how much data it uses so to give them some flavor} 
At each iteration of SQB, the parameter vector is updated as follows:
\[\btheta_{k+1} = \btheta_{k} - \bxi^k,
\]
where $\bxi^k = \alpha(\bSigma_{\sS}^k + \eta \bI)^{-1}(\bmu_{\sT}^k + \eta \btheta^k)$ ($\alpha$ is the step size; we use constant step size for SQB in our experiments). Notice that $\bxi^k$ is the solution to the linear system $(\bSigma_{\sS}^k + \eta\bI)\bxi^k = \bmu_{\sT}^k + \eta \btheta^k$ and can be efficiently computed using the \textsf{lsqr} solver (or any other iterative solver). For all experiments, we ran a small number ($l$) iterations of \textsf{lsqr}, where $l$ was chosen from the set $\{5,10,20\}$, before updating the parameter vector (this technique may be viewed as performing conjugate gradient on the bound), and chose $l$ with the best performance (the fastest and most stable convergence).

%\textcolor{red}{Describe similar scheme for chosing the mini-batch size for $\Sigma$ estimation in SQB}

\subsection{Step size}
One of the most significant disadvantages of standard stochastic gradient methods~\cite{citeulike:432261,DBLP:conf/nips/BottouC03} is the choice of the step size. Stochastic gradient algorithms can achieve dramatic convergence rate if the step size is badly tuned~\cite{Nemirovski,DBLP:conf/nips/RouxSB12}. 
An advantage of computing updates using approximated curvature terms is that the inversion also establishes a scale for the problem, 
and requires minimal tuning. This is well known (phenomenologically) in inverse problems. In all experiments below, we used a constant step size; 
for well conditioned examples we used step size of $1$, and otherwise $0.1$.

\subsection{Comparator methods}
We compared SQB method with the variety of competitive state-of-the-art methods which we list below:
\begin{itemize}
\item \textbf{L-BFGS}: limited-memory BFGS method (quasi-Newton method) tuned for log-linear models which uses both first- and second-order information about the objective function (for L-BFGS this is gradient and approximation to the Hessian); we use the competitive implementation obtained from \textsf{\small{http://www.di.ens.fr/~mschmidt/Software/minFunc.html}}
\item \textbf{SGD}: stochastic gradient descent method with constant step size; we use the competitive implementation obtained from \textsf{\small{http://www.di.ens.fr/~mschmidt/Software/SAG.html}} which is analogous to L. Bottou implementation but with pre-specified step size
\item \textbf{ASGD}: averaged stochastic gradient descent method with constant step size; we use the competitive implementation obtained from \textsf{http://www.di.ens.fr/~mschmidt/Software/SAG.html} which is analogous to L. Bottou implementation but with pre-specified step size
\item \textbf{SAG}: stochastic average gradient method using the estimate of Lipschitz constant $L_k$ at iteration $k$ set constant to the global Lipschitz constant; we use the competitive implementation of~\cite{DBLP:conf/nips/RouxSB12} obtained from \textsf{\small{http://www.di.ens.fr/~mschmidt/Software/SAG.html}}
\item \textbf{SAGls}: stochastic average gradient method with line search, we use the competitive implementation of~\cite{DBLP:conf/nips/RouxSB12} obtained from \textsf{\small{http://www.di.ens.fr/~mschmidt/Software/SAG.html}}; the algorithm adaptively estimates Lipschitz constant $L$ with respect to the logistic loss function using line-search

Since our method uses the constant step size we chose to use the same scheme for the competitor methods like SGD, ASGD and SAG. For those methods we tuned the step size to achieve the best performance (the fastest and most stable convergence). Remaining comparators (L-BFGS and SAGls) use line-search.
\end{itemize}

\section{Experiments}
\label{sec:convex}

\begin{figure}[t]
  \center
\includegraphics[width = 1.81in]{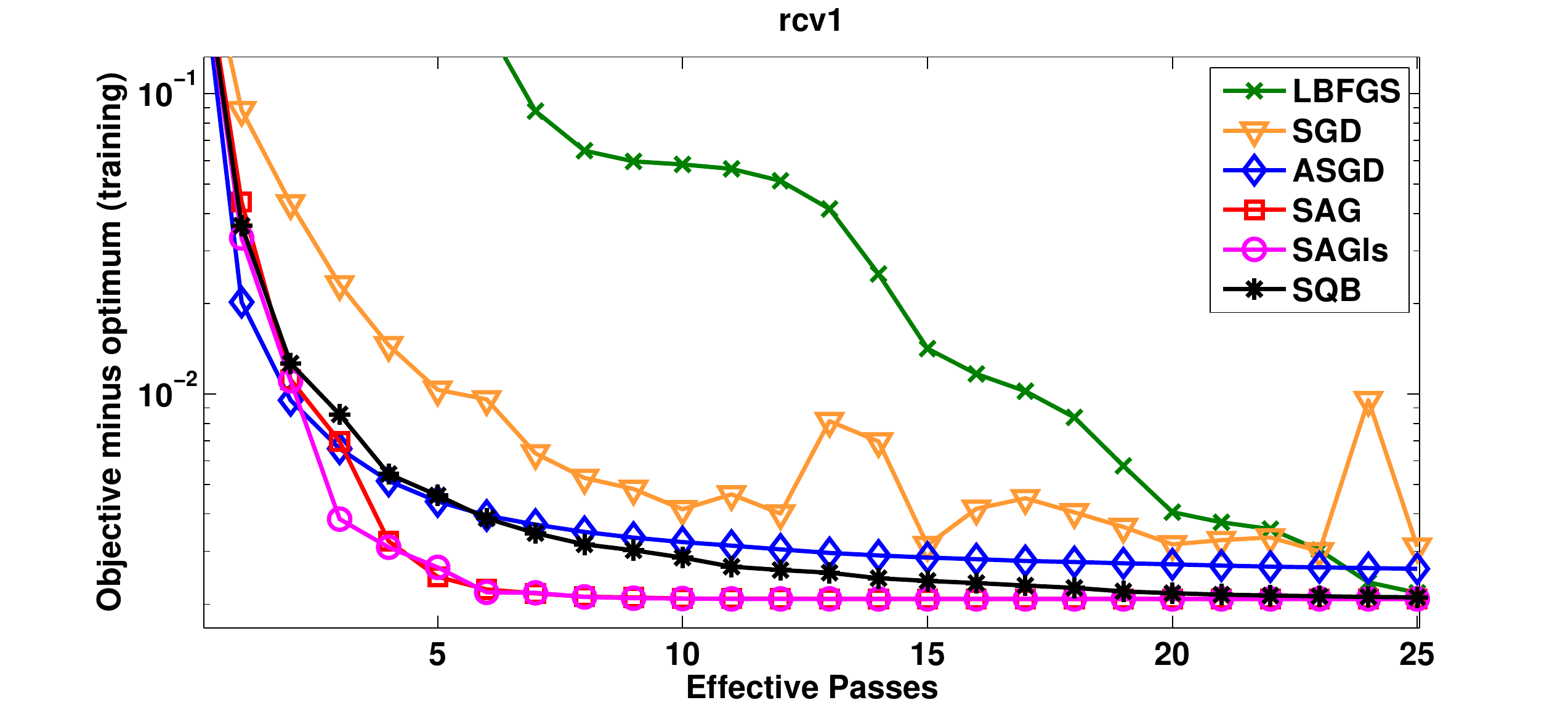}
\includegraphics[width = 1.81in]{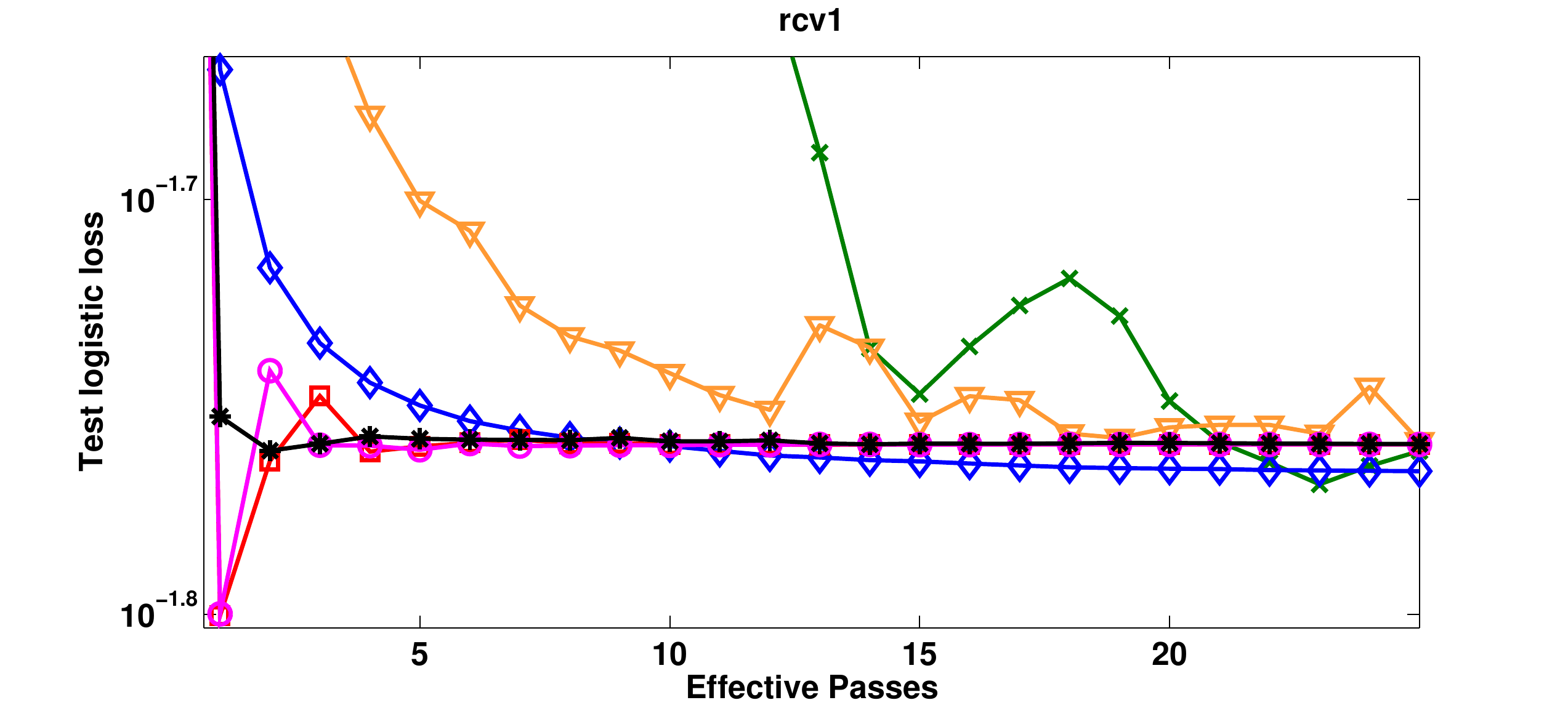}
\includegraphics[width = 1.81in]{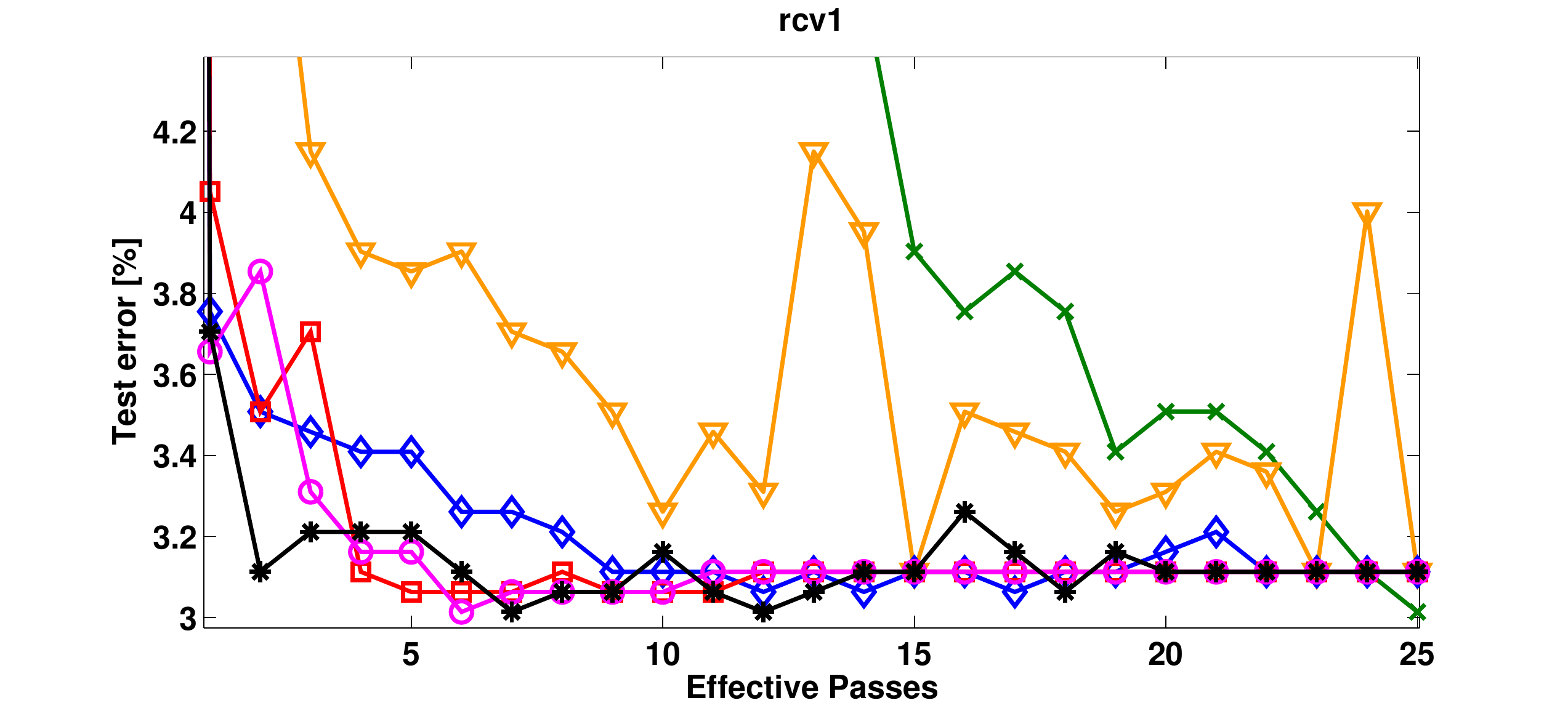}\\
\includegraphics[width = 1.81in]{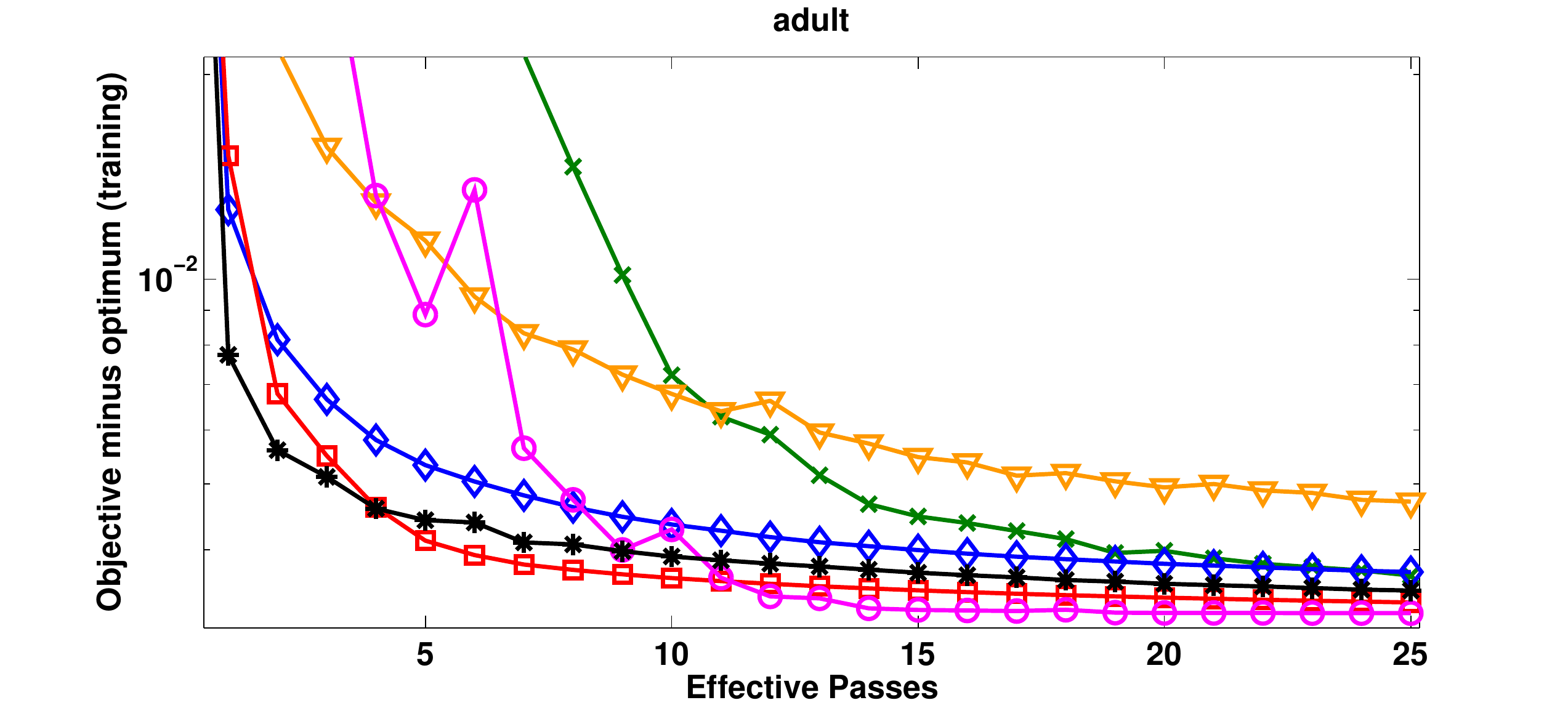}
\includegraphics[width = 1.81in]{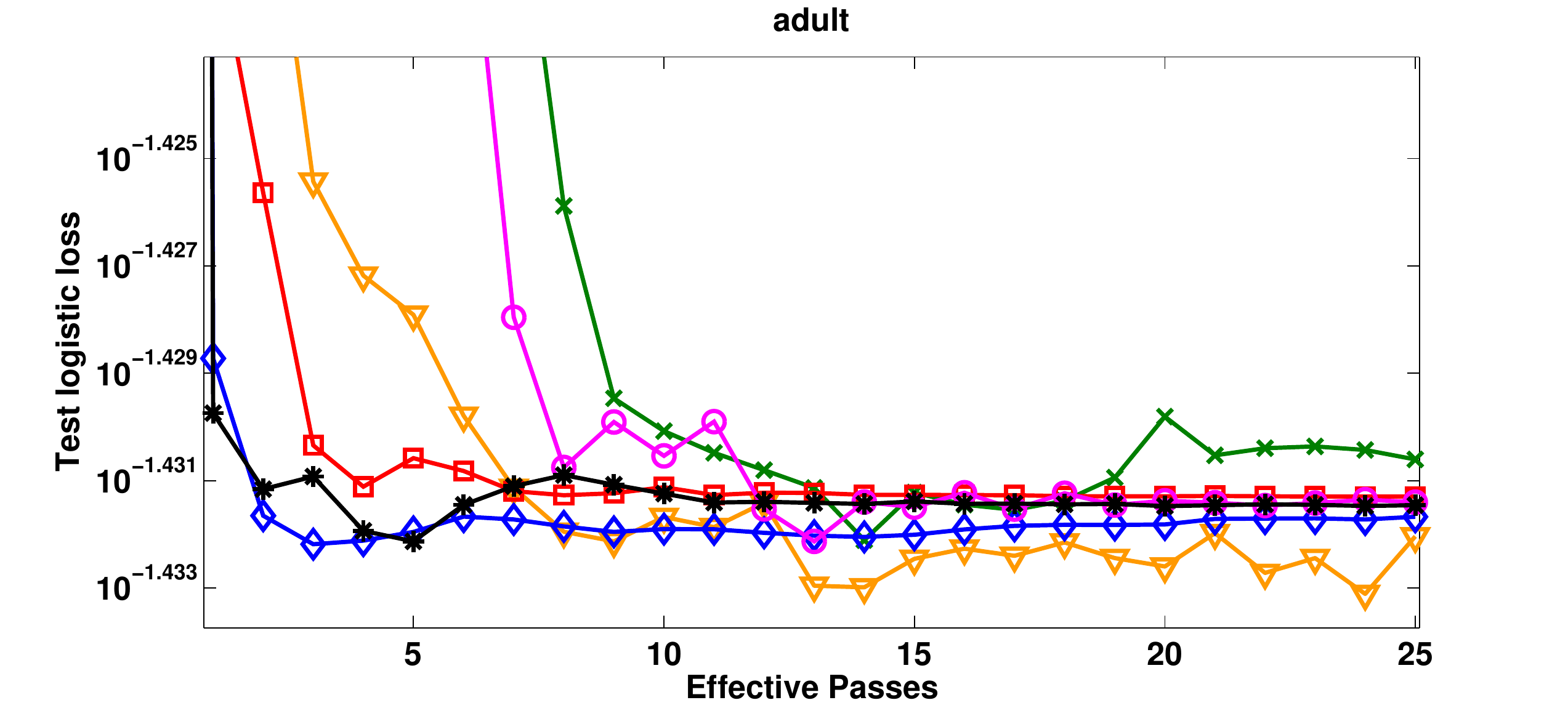}
\includegraphics[width = 1.81in]{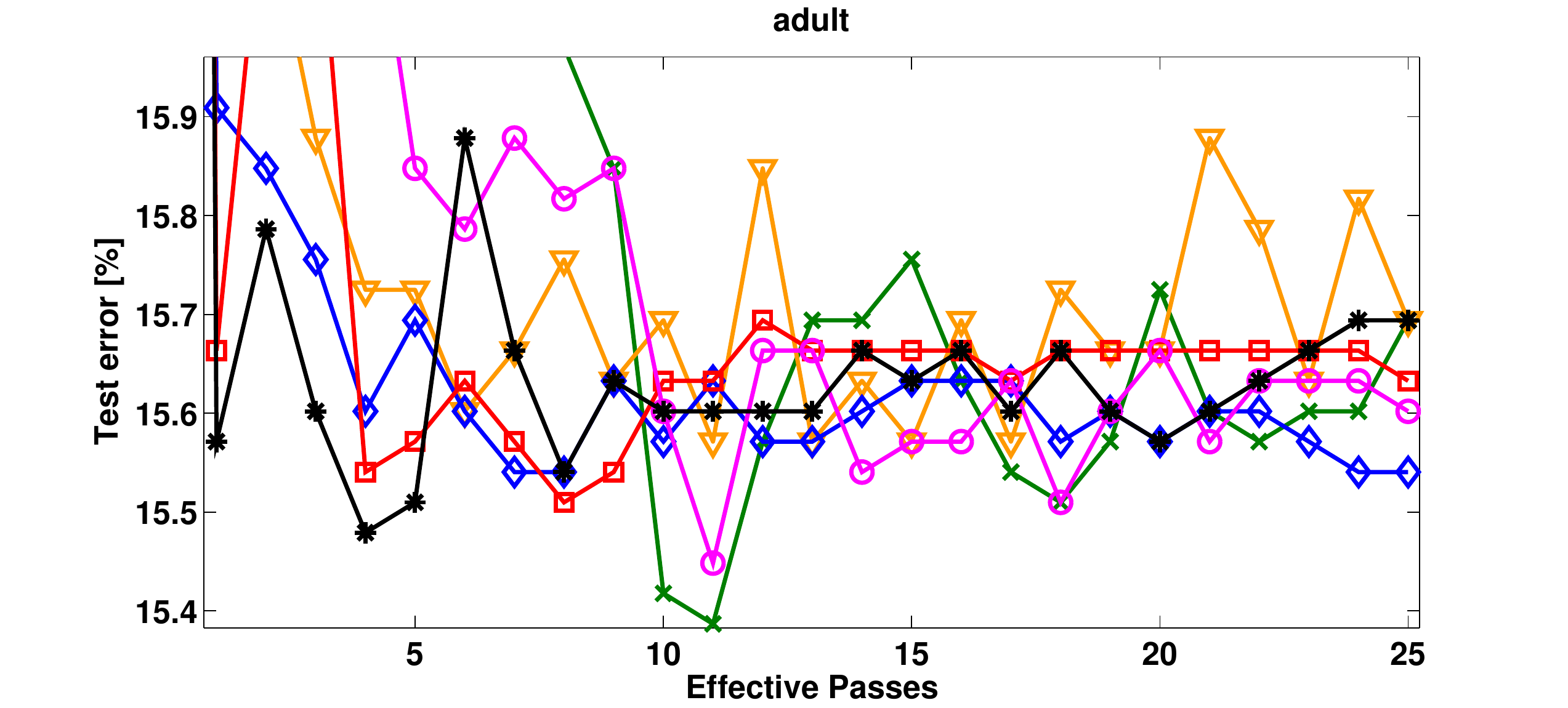}\\
\includegraphics[width = 1.81in]{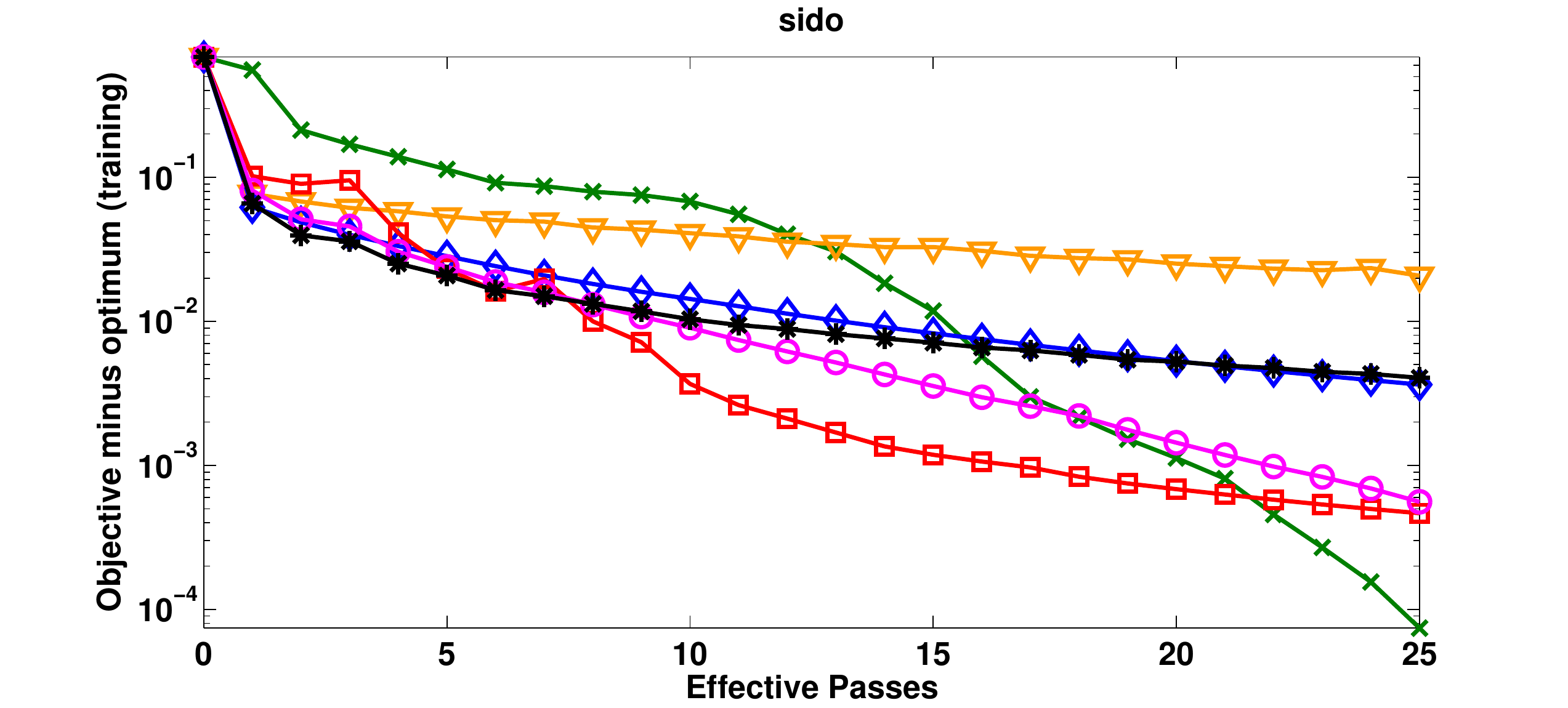}
\includegraphics[width = 1.81in]{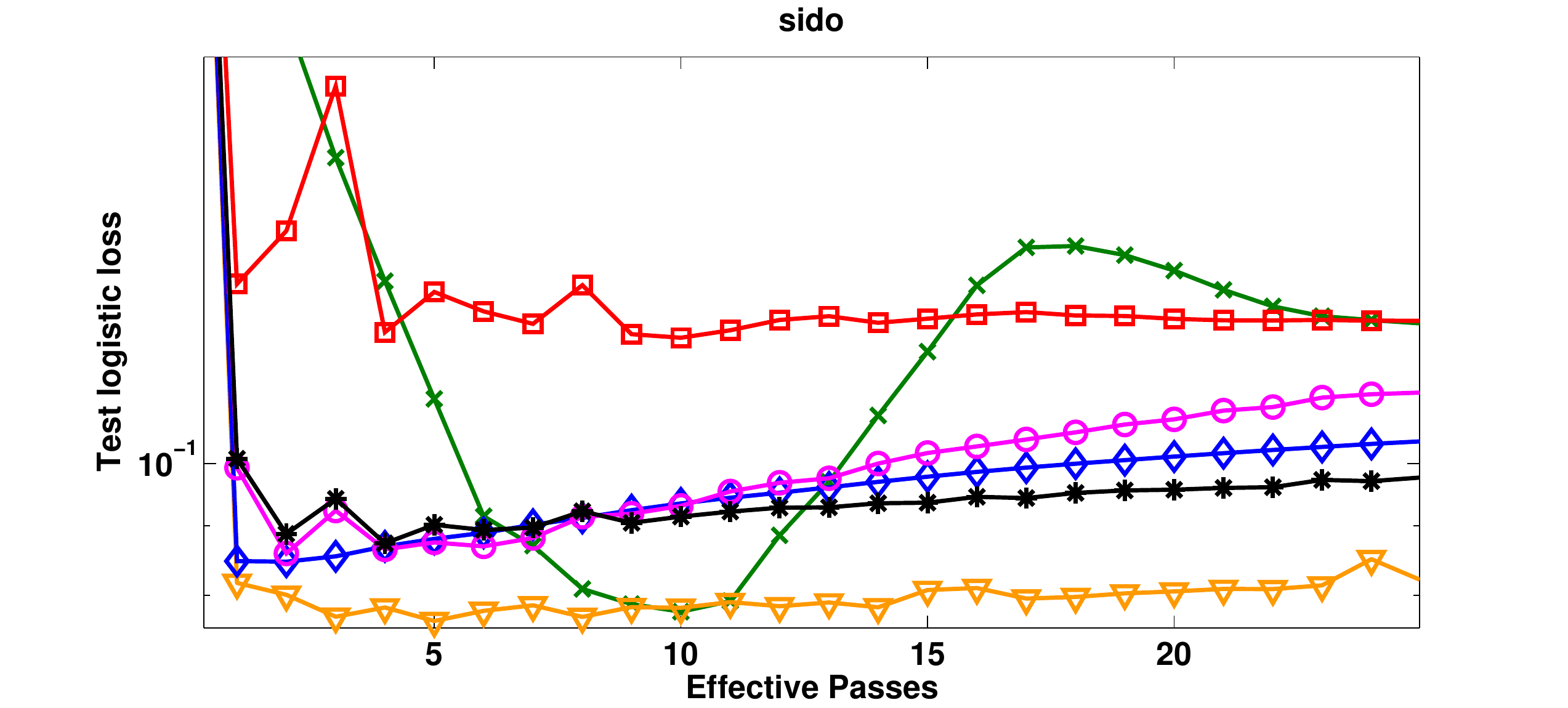}
\includegraphics[width = 1.81in]{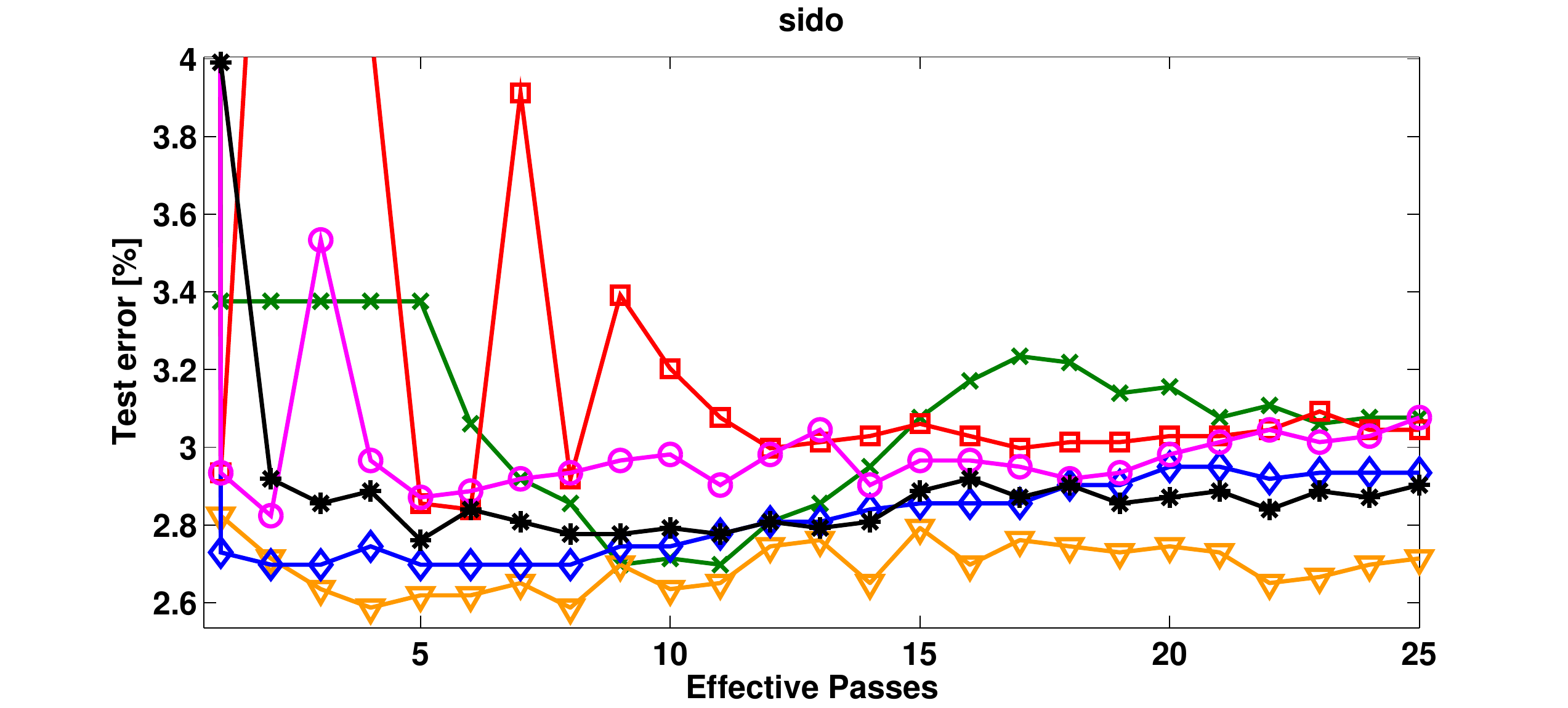}\\
\includegraphics[width = 1.81in]{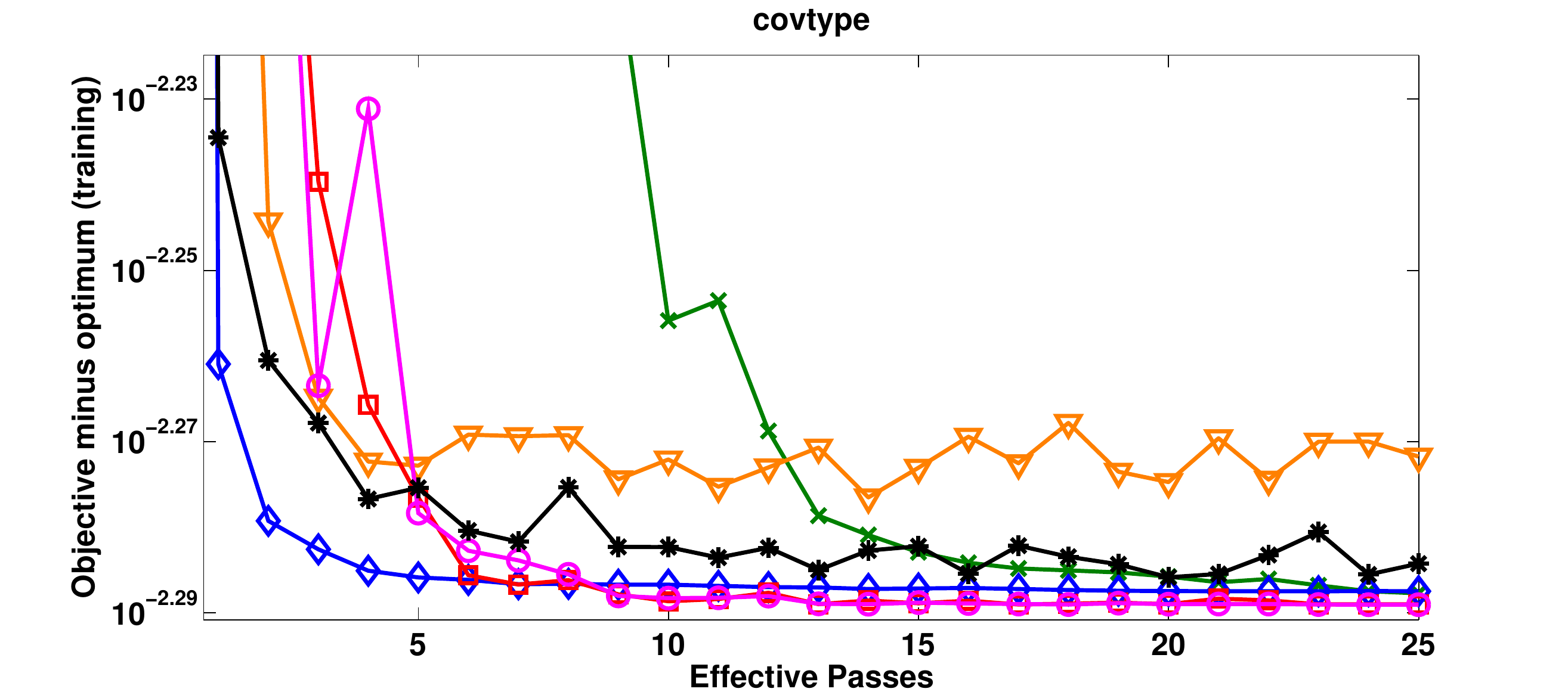}
\includegraphics[width = 1.81in]{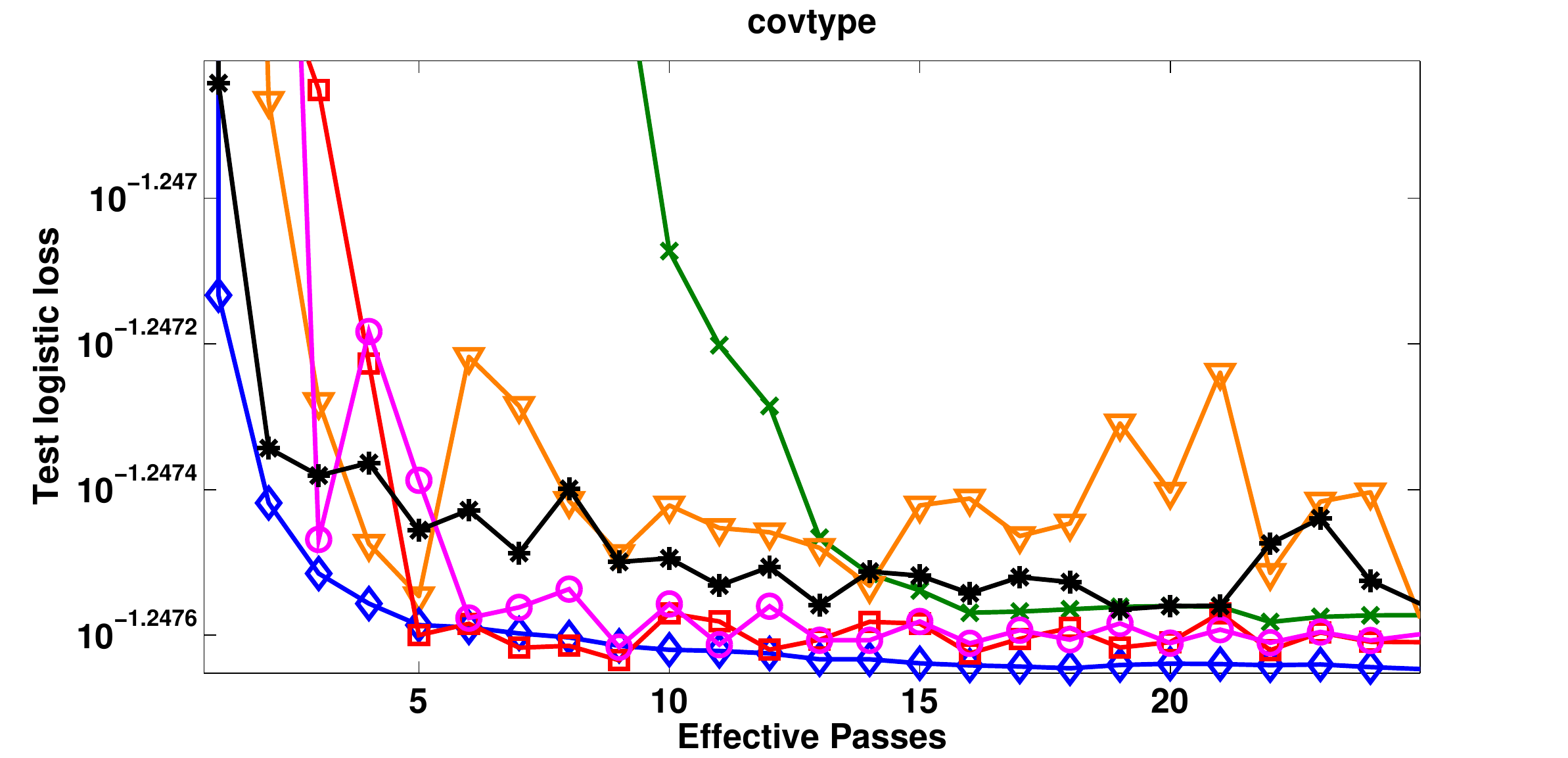}
\includegraphics[width = 1.81in]{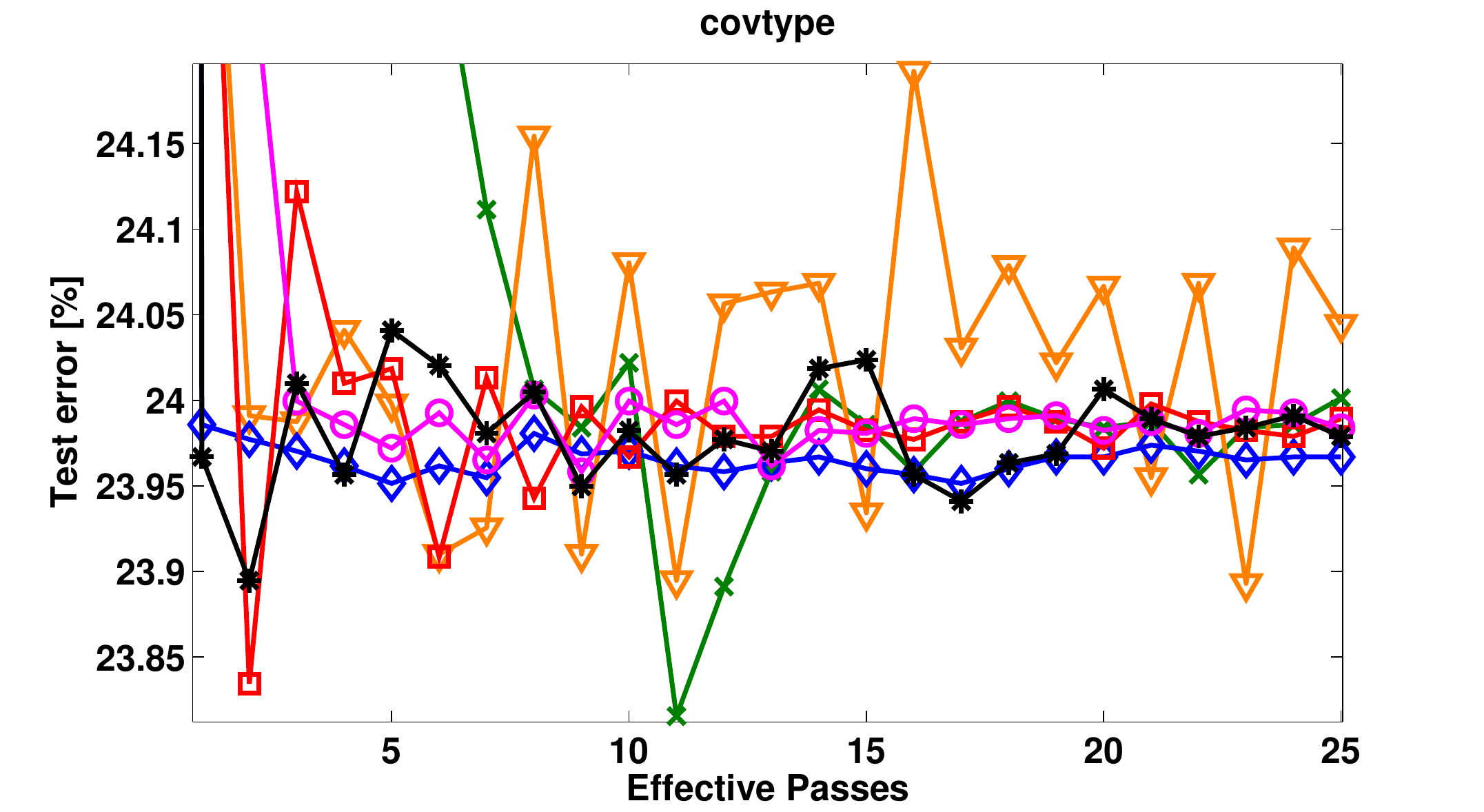}\\
\includegraphics[width = 1.81in]{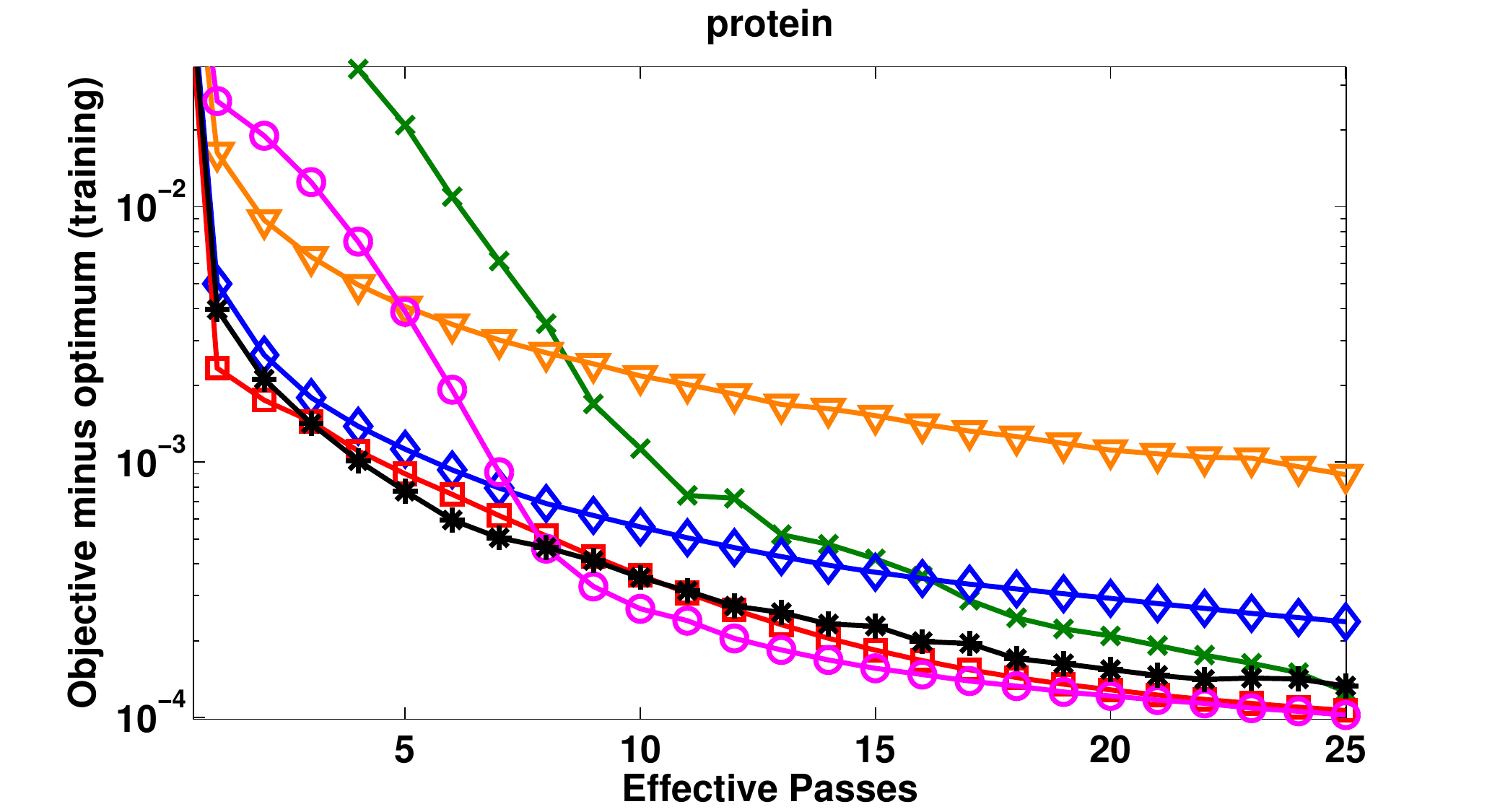}
\includegraphics[width = 1.81in]{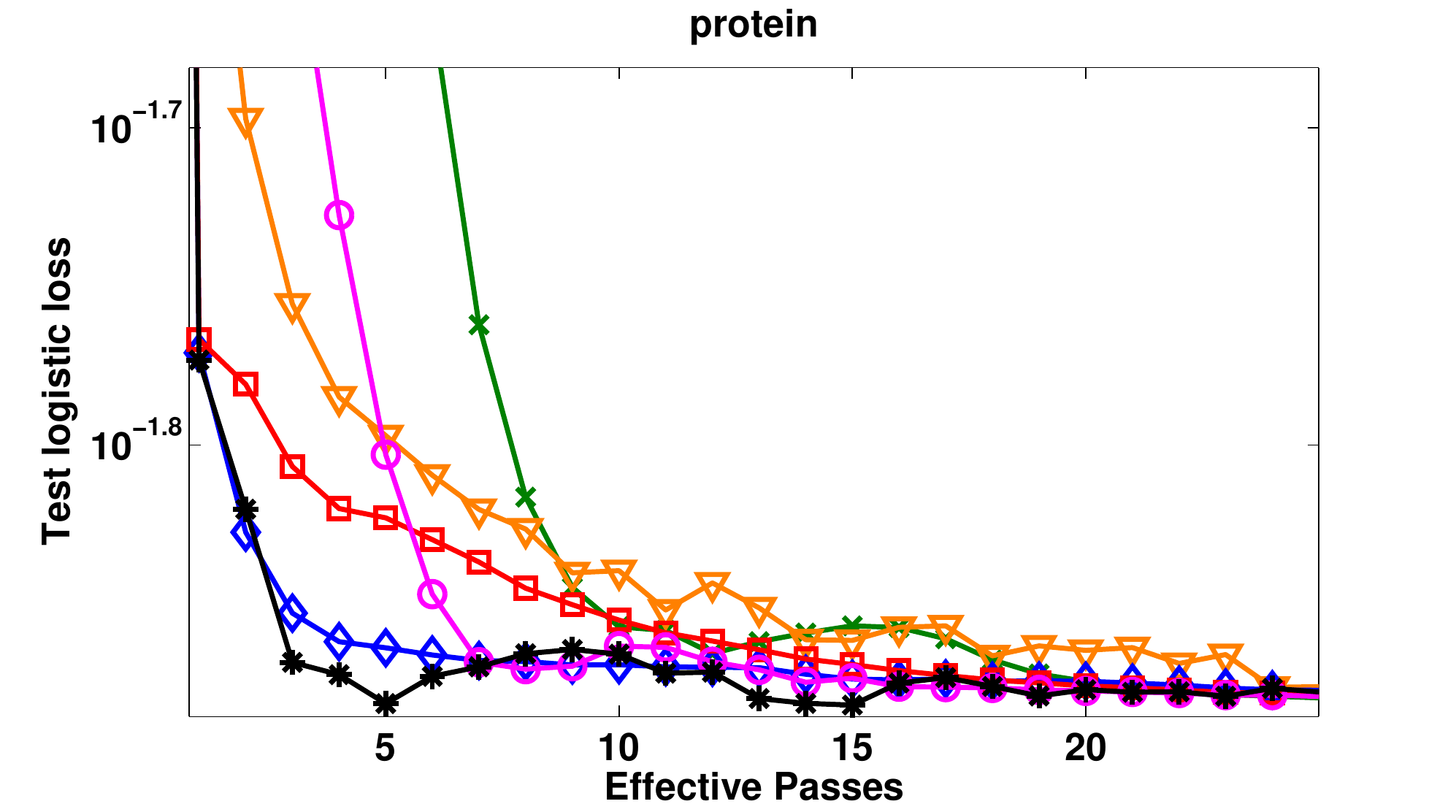}
\includegraphics[width = 1.81in]{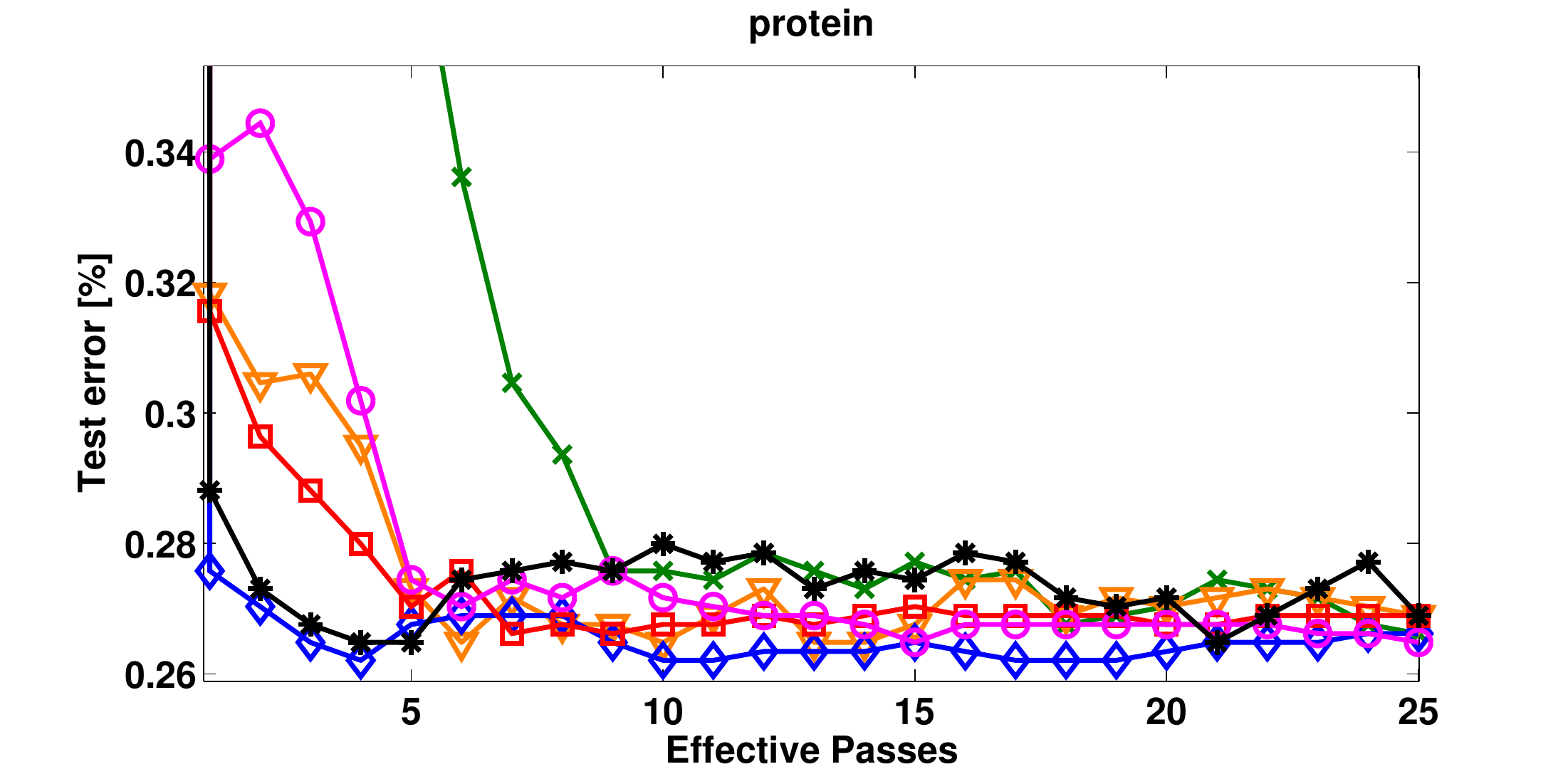}\\
\includegraphics[width = 1.81in]{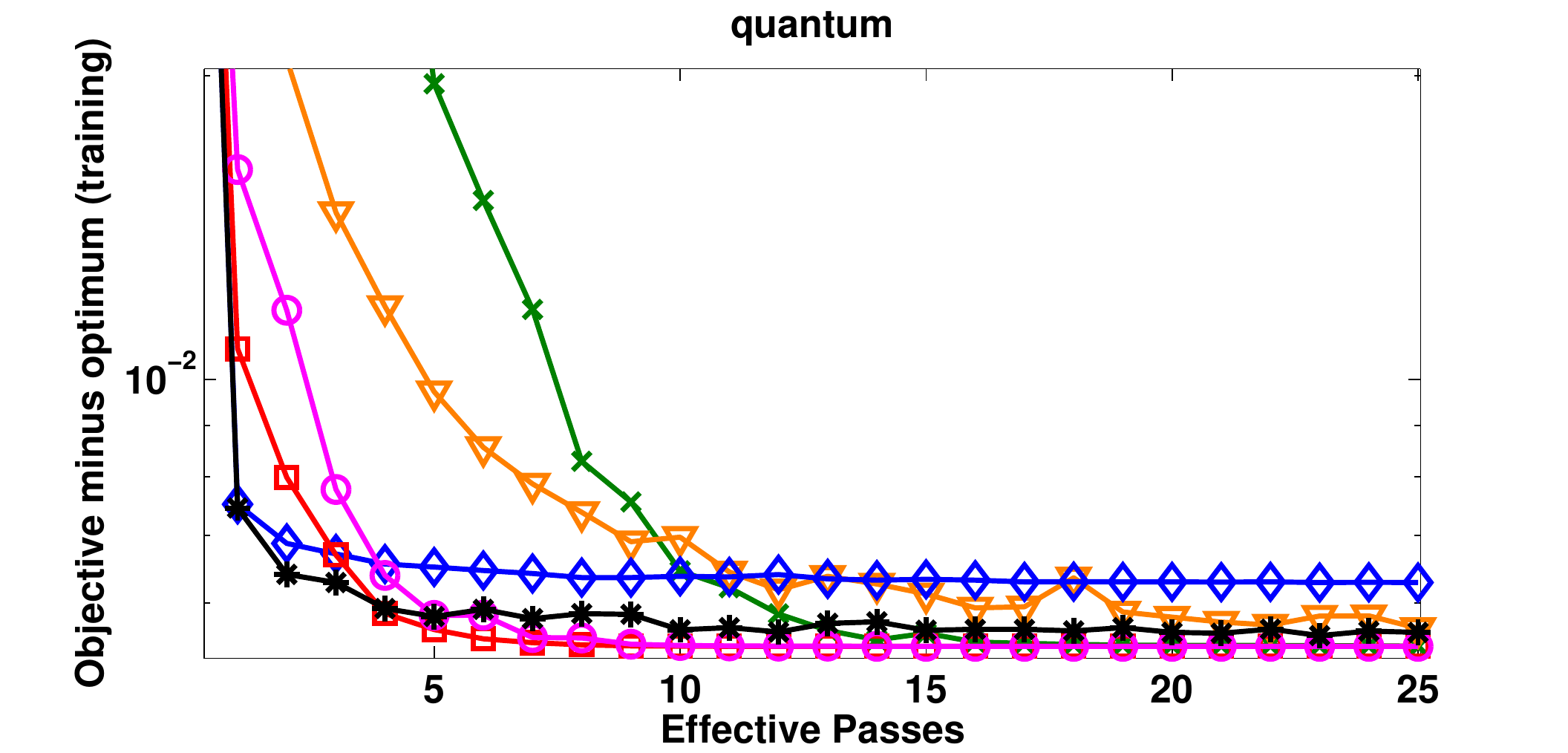}
\includegraphics[width = 1.81in]{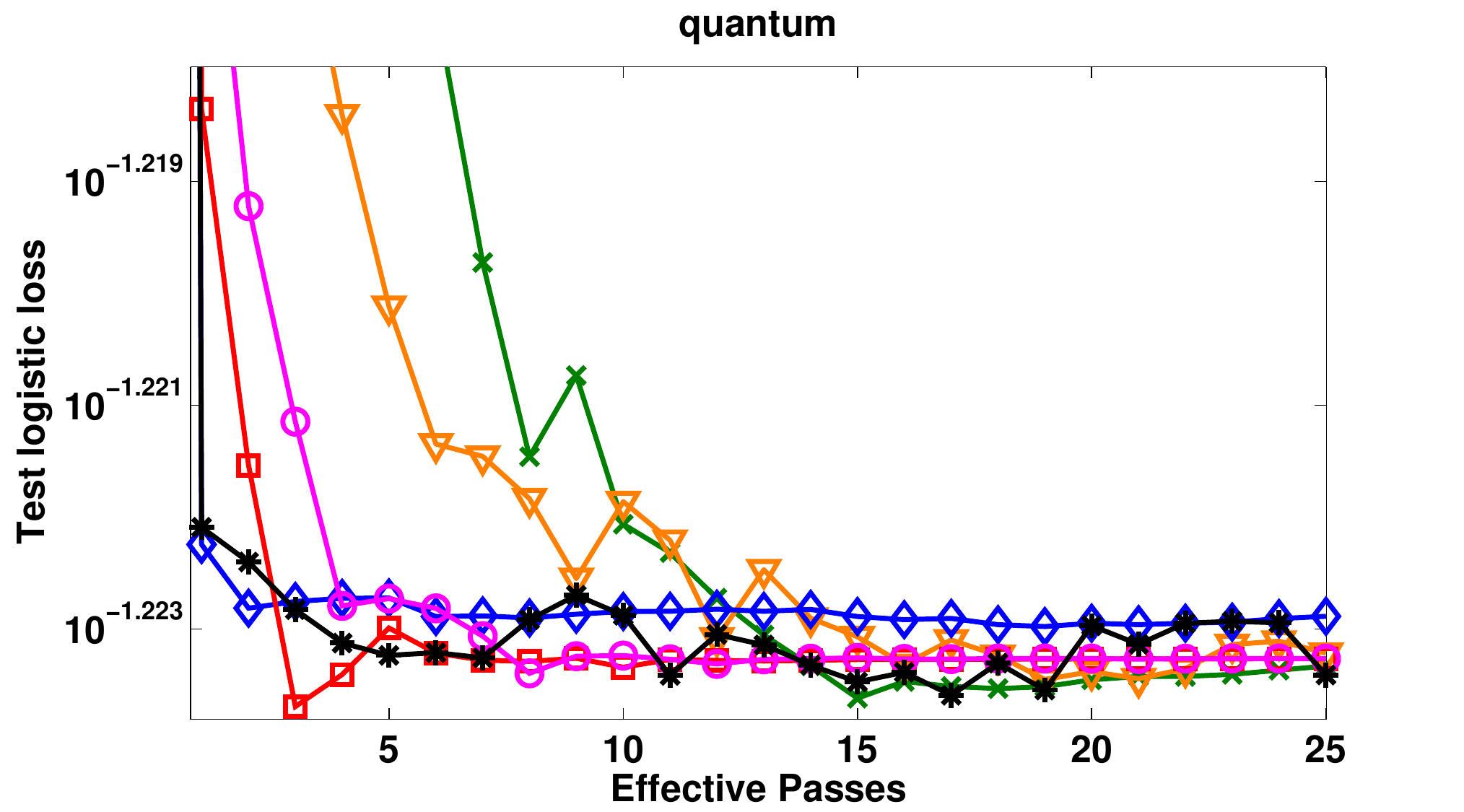}
\includegraphics[width = 1.81in]{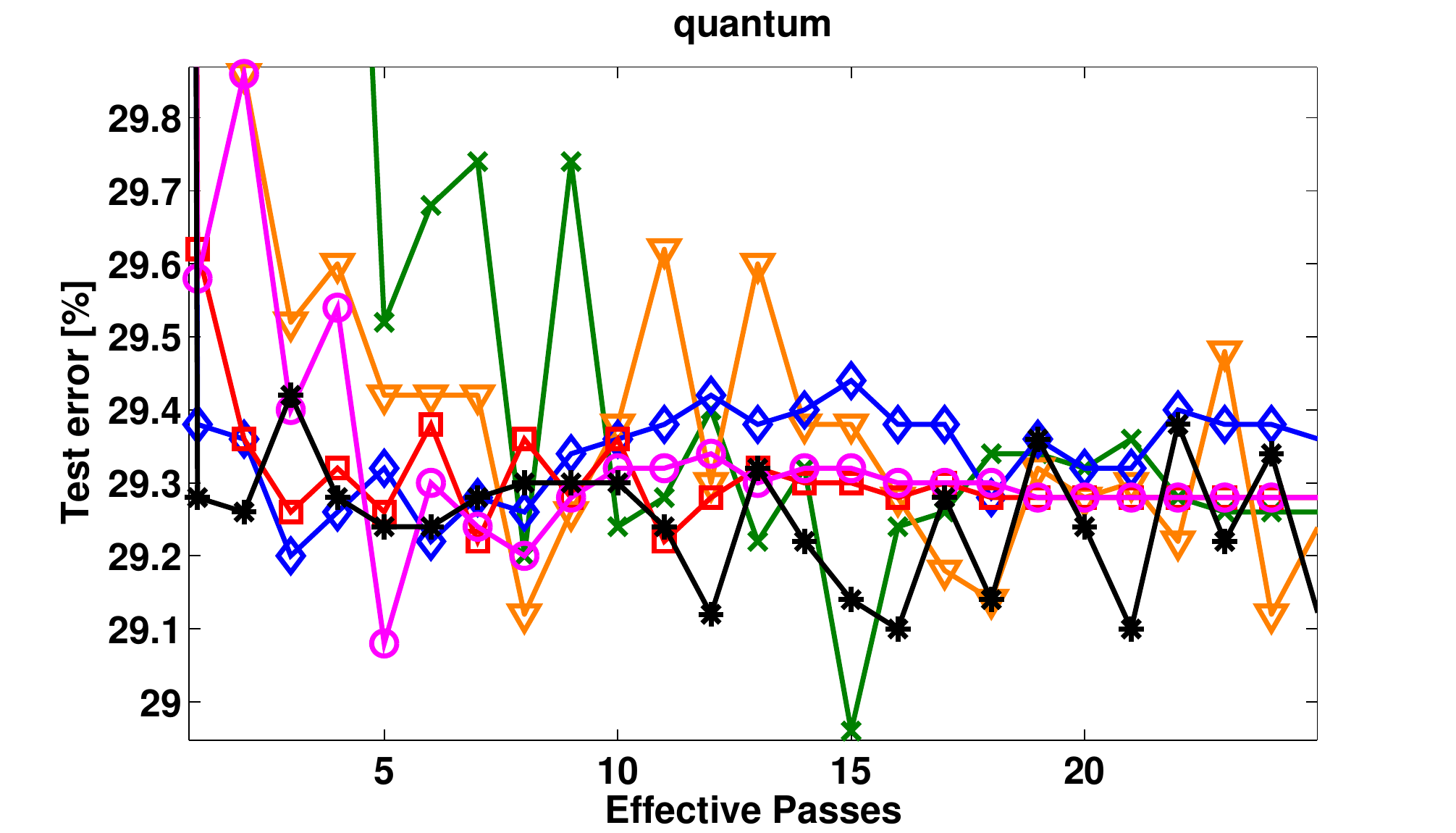}\\
\caption{Comparison of optimization strategies for $l_2$-regularized logistic regression. From left two right: training excess cost, testing cost and testing error. From top to bottom: \textit{rcv1} ($\alpha_{\text{SGD}} = 10^{-1}$, $\alpha_{\text{ASGD}} = 1$, $\alpha_{\text{SQB}} = 10^{-1}$), \textit{adult} ($\alpha_{\text{SGD}} = 10^{-3}$, $\alpha_{\text{ASGD}} = 10^{-2}$, $\alpha_{\text{SQB}} = 1$), \textit{sido} ($\alpha_{\text{SGD}} = 10^{-3}$, $\alpha_{\text{ASGD}} = 10^{-2}$, $\alpha_{\text{SQB}} = 1$), \textit{covtype} ($\alpha_{\text{SGD}} = 10^{-4}$, $\alpha_{\text{ASGD}} = 10^{-3}$, $\alpha_{\text{SQB}} = 10^{-1}$), \textit{protein} ($\alpha_{\text{SGD}} = 10^{-3}$, $\alpha_{\text{ASGD}} = 10^{-2}$, $\alpha_{\text{SQB}} = 1$) and \textit{quantum} ($\alpha_{\text{SGD}} = 10^{-4}$, $\alpha_{\text{ASGD}} = 10^{-2}$, $\alpha_{\text{SQB}} = 10^{-1}$) datasets. This figure is best viewed in color.} 
\label{fig:regression}
\end{figure}

We performed experiments with $l_2$-regularized logistic regression on binary classification task with regularization parameter $\eta = \frac{1}{T}$. We report the results for six datasets.
The first three are sparse: \textit{rcv1} ($T = 20242$, $d = 47236$; SQB parameters: $l = 5$, $\gamma_{\bmu} = 0.005$, $\gamma_{\bSigma} = 0.0003$), \textit{adult} ($T = 32561$, $d = 123$; SQB parameters: $l = 5$, $\gamma_{\bmu} = 0.05$, $\gamma_{\bSigma} = 0.001$) and \textit{sido} ($T = 12678$, $d = 4932$; SQB parameters: $l = 5$, $\gamma_{\bmu} = 0.01$, $\gamma_{\bSigma} = 0.0008$). The remaining datasets are dense: \textit{covtype} ($T = 581012$, $d = 54$; SQB parameters: $l = 10$, $\gamma_{\bmu} = 0.0005$, $\gamma_{\bSigma} = 0.0003$), \textit{protein} ($T = 145751$, $d = 74$; SQB parameters: $l = 20$, $\gamma_{\bmu} = 0.005$, $\gamma_{\bSigma} = 0.001$) and \textit{quantum} ($T = 50000$, $d = 78$; SQB parameters:  $l = 5$, $\gamma_{\bmu} = 0.001$, $\gamma_{\bSigma} = 0.0008$). Each dataset was split to training and testing datasets such that $90\%$ of the original datasets was used for training and the remaining part for testing. Only \textit{sido} and \textit{protein} were split in half to training and testing datasets due to large disproportion of the number of datapoints belonging to each class. The experimental results we obtained are shown in Figure~\ref{fig:regression}. We report the training and testing costs as well as the testing error as a function of the number of effective passes through the data and thus the results do not rely on the implementation details. We would like to emphasize however that under current implementation the average running time for the bound method across the datasets is comparable to that of the competitor methods. All codes are released and are publicly available at \textsf{\small{www.columbia.edu/~aec2163/NonFlash/Papers/Papers.html}}.

\section{Conclusions}
We have presented a new semistochastic quadratic bound (SQB) method, 
together with convergence theory and several numerical examples. 
The convergence theory is divided into two parts. First, we proved convergence to stationarity of the method under 
 weak hypotheses (in particular, convexity is not required). Second, for the logistic regression problem, we provided 
a stronger convergence theory, including a rate of convergence analysis.  

The main contribution of this paper is to  apply SQB methods in a semi-stochastic large-scale setting. 
In particular, we developed and analyzed a flexible framework that allows sample-based approximations of the bound from~\cite{JebCho12}
that are appropriate in the large-scale setting, computationally efficient, and competitive with state-of-the-art methods.

Future work includes developing a fully stochastic version of SQB, as well as applying it to learn mixture models and other latent models, 
as well as models that induce representations, in the context of deep learning.

%\newpage
\begin{spacing}{0.9}
\bibliographystyle{unsrt}
\small{
\bibliography{SemiBound}
}
\end{spacing}

\end{document}